\pdfoutput=1
\documentclass[conference
]{IEEEtran}

\usepackage[numbers, sort&compress]{natbib}
\usepackage{multicol}
\usepackage[bookmarks=true]{hyperref}

\usepackage{irom}

\usepackage{amsthm}
\usepackage{algorithmic}
\usepackage{algorithm}
\usepackage{booktabs}
\usepackage{breakcites}
\usepackage{thm-restate}
\usepackage{balance}
\usepackage{bbm}
\usepackage{verbatim}
\usepackage{comment}
\usepackage{url}
\usepackage{natbib}

\usepackage[protrusion=true,
    activate={true,nocompatibility},
    final,
    tracking=true,
    kerning=true,
    spacing=false,
    factor=1100]{microtype}
    
\SetTracking{encoding={*}, shape=sc}{10}

\renewcommand{\S}{\mathcal{S}}
\newcommand{\A}{\mathcal{A}}
\renewcommand{\O}{\mathcal{O}}
\renewcommand{\KL}{\mathbb{D}}

\newcommand{\fb}{\mathbb{D}_{f,\mathcal{B}}}
\newcommand{\finv}{\mathbb{D}_f^{-1}}
\newcommand{\I}{\mathbb{I}}
\newcommand{\revision}[1]{\textcolor{black}{#1}}
\newcommand{\revisionfig}{}

\theoremstyle{plain} 
\newtheorem{theorem}{Theorem}

\newtheorem{remark}{Remark}

\newtheorem{definition}{Definition}

\newtheorem{assumption}{Assumption}

\newcommand\BibTeX{{\rmfamily B\kern-.05em \textsc{i\kern-.025em b}\kern-.08em
T\kern-.1667em\lower.7ex\hbox{E}\kern-.125emX}}

\begin{document}


\title{Fundamental Limits for Sensor-Based \\ Robot Control}


\author{\authorblockN{Anirudha Majumdar, Zhiting Mei, and Vincent Pacelli} 
\authorblockA{\small{Department of Mechanical and Aerospace Engineering}\\
\small{Princeton University, Princeton, NJ, 08540}\\
\small{Emails: \{ani.majumdar, zm2074, vpacelli\}@princeton.edu}}}

\maketitle

\begin{abstract}
Our goal is to develop theory and algorithms for establishing \emph{fundamental limits} on performance imposed by a robot's sensors for a given task. In order to achieve this, we define a quantity that captures the amount of \emph{task-relevant information} provided by a sensor. Using a novel version of the generalized Fano inequality from information theory, we demonstrate that this quantity provides an upper bound on the highest achievable expected reward for one-step decision making tasks. We then extend this bound to multi-step problems via a dynamic programming approach. We present algorithms for numerically computing the resulting bounds, and demonstrate our approach on three examples: (i) the lava problem from the literature on partially observable Markov decision processes, (ii) an example with continuous state and observation spaces corresponding to a robot catching a freely-falling object, and (iii) obstacle avoidance using a depth sensor with non-Gaussian noise. We demonstrate the ability of our approach to establish strong limits on achievable performance for these problems by comparing our upper bounds with achievable lower bounds (computed by synthesizing or learning concrete control policies).  
\end{abstract}

\IEEEpeerreviewmaketitle

\vspace{-5pt}
\section{Introduction}
\label{sec:intro}


Robotics is often characterized as the problem of transforming ``pixels to torques" \citep{Brock11}: how can an embodied agent convert raw sensor inputs into actions in order to accomplish a given task? In this paper, we seek to understand the \emph{fundamental limits} of this process by studying the following question: is there an \emph{upper bound} on performance imposed by the sensors that a robot is equipped with? 

As a motivating example, consider the recent debate around the ``camera-only" approach to autonomous driving favored by Tesla versus the ``sensor-rich" philosophy pursued by Waymo~\citep{Morris21}. Is an autonomous vehicle equipped only with cameras \emph{fundamentally limited} in terms of the \revision{collision rate} it can achieve? 
By ``fundamental limit", we mean a bound on performance \revision{or safety} on a given task that holds \emph{regardless} of the form of control policy one utilizes (e.g., a neural network with billions of parameters \revision{or} a nonlinear model predictive control scheme combined with a particle filter), how the policy is synthesized (e.g., via model-free reinforcement learning \revision{or} model-based control), or how much computation is available to the robot or software designer. 

While there have been tremendous algorithmic advancements in robotics over decades, we currently lack a ``science" for understanding such fundamental limits \citep{Koditschek21}.
Current practice in robotics is often \emph{empirical} in nature. \revision{For example, practitioners today often implement a process of trial and error with different perception and control architectures that use neural networks of varying sizes and architectures}. Techniques for establishing fundamental limits imposed by a sensor would potentially allow us to glean important design insights \revision{such as} realizing that a particular sensor is not sufficient for a task and must be replaced. Further, such techniques could allow us to establish the superiority of one suite of sensors over another from the perspective of a given task; \revision{this may be achieved} by synthesizing a control policy for one sensor suite that achieves better performance than the fundamental bound for another suite.

\begin{figure}[t]
\revisionfig
\begin{center}
\vspace{-3pt}
\includegraphics[width=0.99\columnwidth]{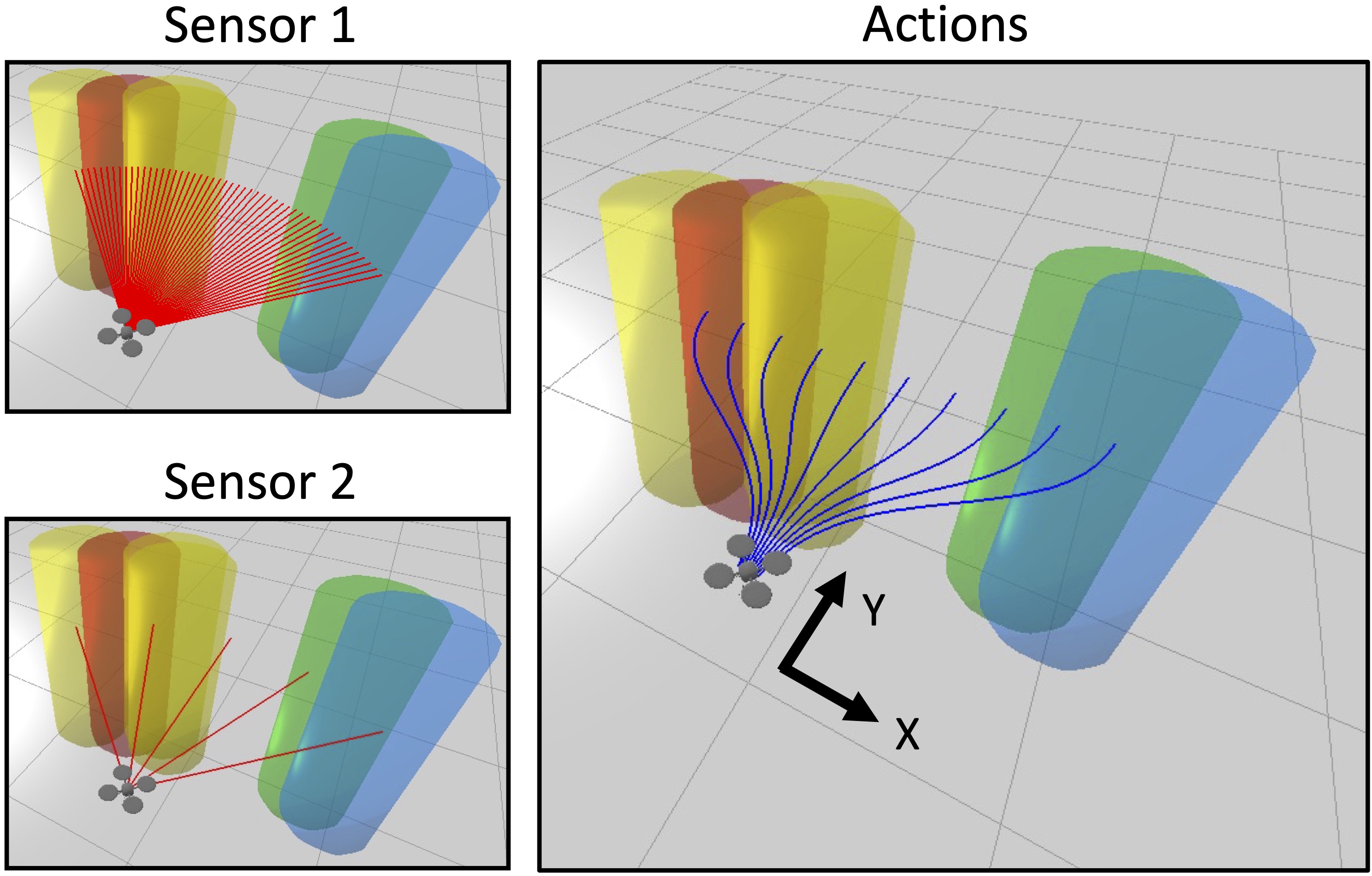}
\end{center} 
\vspace{-5pt}
\caption{\revision{Our goal is to establish fundamental limits on performance for a given task imposed by a robot's sensors. We propose a quantity that captures the amount of \emph{task-relevant information} provided by a sensor and use it to bound the highest achievable performance (expected reward) using the sensor. We demonstrate our approach on examples that include obstacle avoidance with a noisy depth sensor (left figure) using motion primitives (right figure). Our framework also allows us to establish the superiority of one sensor (Sensor 1: a dense depth sensor) over another (Sensor 2: a sparse depth sensor) for a given task.    \label{fig:anchor}}}
\vspace{-10pt}
\end{figure}

In this paper, we take a step towards \revision{these goals}. We \revision{consider settings where the robot's task performance or safety may be quantified via a reward function.} We then observe that
any technique for establishing fundamental bounds on performance imposed by a given sensor must take into account two factors: (i) the \emph{quality} of the sensor \revision{as measured by} the amount of information about the state of the robot and its environment provided by the sensor, and, importantly, (ii) the \emph{task} that the robot is meant to accomplish. As an example, consider a drone equipped with a (noisy) depth sensor (Figure~\ref{fig:anchor}).
Depending on the nature of the task, the robot may need more or less information from its sensors. 
For example, suppose that the obstacle locations are highly constrained such that a particular sequence of actions always succeeds in avoiding them (i.e., there is a purely open-loop policy that achieves good performance on the task); in this case, even an extremely noisy or sparse depth sensor allows the robot to perform well. However, if the distribution of obstacles is such that there is no pre-defined gap in the obstacles, then a noisy or sparse depth sensor may fundamentally limit the achievable performance on the task. The achievable performance is thus intuitively influenced by the amount of \emph{task-relevant information} provided by the robot's sensors. 

{\bf Statement of contributions.} Our primary contribution is to develop theory and algorithms for establishing fundamental bounds on performance imposed by a robot's sensors for a given task. Our key insight is to define a quantity that captures the \emph{task-relevant information} provided by the robot's sensors. Using a novel version of the \emph{generalized Fano's inequality} from information theory, we demonstrate that this quantity provides a fundamental upper bound on expected reward for one-step decision making problems. We then extend this bound to multi-step settings via a dynamic programming approach and propose algorithms for computing the resulting bounds for systems with potentially continuous state and observation spaces, nonlinear and stochastic dynamics, and non-Gaussian sensor models (but with discretized action spaces). 
We demonstrate our approach on three examples: (i) the lava problem from the literature on partially observable Markov decision processes (POMDPs), (ii) a robot catching a freely-falling object, and (iii) obstacle avoidance using a depth sensor (Figure \ref{fig:anchor}). We demonstrate the strength of our \revision{upper bounds on performance by comparing them against \emph{lower bounds} on the best achievable performance obtained from concrete control policies}: the optimal POMDP solution for the lava problem, a model-predictive control (MPC) scheme for the catching example, and a learned neural network policy \revision{and heuristic planner} for the obstacle avoidance problem. We also present applications of our approach for establishing the superiority of one sensor over another from the perspective of a given task.  To our knowledge, the results in this paper are the first to provide general-purpose techniques for establishing fundamental bounds on performance for sensor-based control of robots. 

A preliminary version of this work was published in the proceedings of the Robotics: Science and Systems (RSS) conference \citep{Majumdar22}. In this significantly extended and revised version, we additionally present: (i) the extension of the definition of the task-relevant information potential (TRIP) from using KL-divergence only to the more general $f$-divergence (Section \ref{sec:one-step}), (ii) the generalization of the single-step performance bound with the extended TRIP definition (Theorem \ref{thm:single-step bound}), (iii) the generalization of the upper bound of performance for multi-step problems (Theorem \ref{thm:multi-step bound}), (iv) the application of the generalized bounds to the multi-step lava problem (Section \ref{sec:lava problem}), (v) a method for optimizing the upper bound by varying the function used to define the $f$-divergence in the multi-step lava problem (Section \ref{sec:lava problem}), and (vi) results demonstrating that our novel version of the generalized Fano's inequality results in tighter bounds as compared to the original generalized Fano's inequality \citep{Chen16, Gerchinovitz20} (Section \ref{sec:lava problem}).

\subsection{Related Work}
\label{sec:related work}

{\bf Domain-specific performance bounds.} 
Prior work in robotics has established fundamental bounds on performance for particular problems. For example, \cite{Karaman12} and \cite{Choudhury15} consider high-speed navigation through an ergodic forest consisting of randomly-placed obstacles. Results from percolation theory~\citep{Bollobas06} are used to establish a critical speed beyond which there does not exist (with probability one) an infinite collision-free trajectory. 
The work by \cite{Falanga19} establishes limits on the speed at which a robot can navigate through unknown environments in terms of perceptual latency. Classical techniques from robust control \citep{Doyle13} have also been utilized to establish fundamental limits on performance for control tasks (e.g., pole balancing) involving linear output-feedback control and sensor noise or delays \citep{Leong16}. The results obtained by \cite{Xu21} demonstrate empirical correlation of the complexity metrics presented in the work by \cite{Leong16} with sample efficiency and performance of learned perception-based controllers on a pole-balancing task. The approaches mentioned above consider specific tasks \revision{such as} navigation in ergodic forests, or relatively narrow classes of problems \revision{such as} linear output-feedback control. In contrast, our goal is to develop a general and broadly applicable theoretical and algorithmic framework for establishing fundamental bounds on performance imposed by a sensor for a given task. 

{\bf Comparing sensors.} 
The notion of a \emph{sensor lattice} was introduced by \cite{Lavalle12, Lavalle19} for comparing the power of different sensors. \revision{The works by} \cite{OKane08} and \cite{saberifar_toward_2019} \revision{present} similar approaches for comparing robots, sensors, and actuators. The sensor lattice provides a partial ordering on different sensors based on the ability of one sensor to simulate another. However, most pairs of sensors are \emph{incomparable} using such a scheme. Moreover, the sensor lattice does not establish the superiority of one sensor over another from the perspective of a given task; instead, the partial ordering is based on the ability of one sensor to perform as well as another in terms of filtering (i.e., state estimation). In this paper, we also demonstrate the applicability of our approach for comparing different sensors. However, this comparison is \emph{task-driven}; we demonstrate how one sensor can be proved to be fundamentally better than another from the perspective of a given task, without needing to estimate states irrelevant to the task.

{\bf Fano's inequality and its extensions.} 
In its original form, \emph{Fano's inequality} \citep{Cover99} relates the lowest achievable error of estimating a signal $x$ from an observation $y$ in terms of the noise in the channel that produces observations from signals. In recent years, Fano's inequality has been significantly extended and applied for establishing fundamental limits for various statistical estimation problems, e.g., lower bounding the Bayes and minimax risks for different learning problems \citep{Duchi13, Chen16, Gerchinovitz20}. In this paper, we build on generalized versions of Fano's inequality \citep{Chen16, Gerchinovitz20} in order to obtain fundamental bounds on performance for robotic systems with noisy sensors. On the technical front, we contribute by deriving a stronger version of the generalized Fano's inequalities presented by \cite{Chen16, Gerchinovitz20} by utilizing the \emph{inverse of the $f$-divergence} (Section~\ref{sec:background}) and computing it using \emph{convex programming} \citep[Ch. 4]{Boyd04}. The resulting inequality, which may be of independent interest, allows us to derive fundamental upper bounds on performance for one-step decision making problems. We then develop a dynamic programming approach for recursively applying the generalized Fano inequality in order to derive bounds on performance for multi-step problems. 




\section{Problem Formulation}
\label{sec:problem formulation}

\subsection{Notation}
\label{sec:notation}

We denote sequences by $x_{i:j} := (x_k)_{k = i}^{j}$ for $i \leq j$. \revision{We use abbreviations inf (infimum) and sup (supremum), and also use the abbreviations LHS (left hand side) and RHS (right hand side) for inequalities.} \revision{Conditional distributions are denoted as $p(x|y)$. } Expectations are denoted as $\mathbb E[\cdot]$ with the variable of integration or its measure appearing below it for contextual emphasis, e.g.: $\underset{x}{\mathbb E}[\cdot], \underset{p(x)}{\mathbb E}[\cdot]$. Expectations with multiple random variables are denoted as $\underset{x,y}{\mathbb E}[\cdot]$ or $\underset{p(x), p(y)}{\mathbb E}[\cdot]$, while conditional expectations are denoted as $\underset{x|y}{\mathbb E}[\cdot]$ or $\underset{p(x|y)}{\mathbb E}[\cdot]$.



\subsection{Problem Statement}

We denote the state of the robot and its environment at time-step $t$ by $s_t \in \S$. Let $p_0$ denote the initial state distribution. Let the robot's sensor observation and control action at time-step $t$ be denoted by $o_t \in \O$ and $a_t \in \A$ respectively. Denote the stochastic dynamics of the state by $p_t(s_t|s_{t-1}, a_{t-1})$ and suppose that the robot's sensor is described by $\sigma_t(o_t|s_t)$. \revision{We note that this formulation handles multiple sensors by concatenating the outputs of different sensors into $o_t$}.
\revision{In this work, we assume that the robot's task is prescribed using reward functions $r_0, r_1, \dots, r_{T-1}: \S \times \A \rightarrow \mathbb R$ at each time-step (up to a finite horizon). }\revision{We use $R$ to represent cumulative expected reward over a time horizon denoted with a subscript. In subsequent sections, we use superscript $\star$ to denote optimality, and use superscript $\perp$ to represent reward achieved by open-loop policies.}

\begin{assumption}[Bounded rewards]
\label{ass:bounded rewards}
We assume that rewards are bounded, and without further loss of generality we assume that $r_t(s_t, a_t) \in [0,1], \ \forall s_t \in \S, a_t \in \A, t  \in \{ 0, \dots, T-1\}$. 
\end{assumption}

The robot's goal is to find a potentially time-varying and history-dependent control policy $\pi_t: \O^{t+1} \rightarrow \A$ that \revision{maps observations $o_{0:t}$ to actions in order to} maximize the total expected reward:
\begin{equation}
\label{eq:control problem}
   \revision{R^\star_{0\to T}} := \underset{\pi_{0:T-1}}{\sup} \ \ \underset{\substack{s_{0:T-1} \\ o_{0:T-1}}}{\mathbb{E}}\Bigg{[}\sum_{t=0}^{T-1} r_t(s_t, \pi_t(o_{0:t}))\Bigg{]}.
\end{equation}
{\bf Goal:} Our goal is to \emph{upper bound} the best achievable expected reward \revision{$R^\star_{0\to T}$} for a given sensor $\sigma_{0:T-1}$. 
We note that we are allowing for completely general policies that are arbitrary time-varying functions of the entire history of observations received up to time $t$ (as long as the functions satisfy measurability conditions that ensure the existence of the expectation in \eqref{eq:control problem}). An upper bound on \revision{$R^\star_{0\to T}$} thus provides a fundamental bound on achievable performance that holds regardless of how the policy is parameterized (e.g., via neural networks \revision{or} receding-horizon control architectures)  or synthesized (e.g., via reinforcement learning \revision{or} optimal control techniques). 

\section{Background}
\label{sec:background}

In this section, we briefly introduce some background material that will be useful throughout the paper. 

\subsection{KL Divergence, \texorpdfstring{$f$}{Lg}-Divergence, and \texorpdfstring{$f$}{Lg}-informativity}

The Kullback-Leibler (KL) divergence between two distributions, $p(x)$ and $q(x)$, is defined as:
\begin{equation}
    \KL(p(x) || q(x)) := \underset{p(x)}{\mathbb E} \Bigg{[} \log{\frac{ p(x)}{q(x)}}  \Bigg{]}. \label{eq: KL}
\end{equation}

This definition can be extended to the more general notion of $f$-divergence:
\begin{equation}
    \KL_f (p(x) \| q(x)) := \underset{q(x)}{\mathbb E} \left[ f\left(\frac{p(x)}{q(x)}\right)  \right], \label{eq: f}
\end{equation}
where $f$ is a convex function on $\mathbb R^+$, and $f(1)=0$. The KL divergence is a special case of the $f$-divergence with $f(x) = x\log x$. 


The $f$-informativity \citep{csiszar_class_1972} between two random variables is defined as:
\begin{align}
    \I_f(x;y) &:= \inf_{q(y)} \underset{p(x)}{\mathbb E} \Bigg{[} \KL_f\big(p(y|x) \| q(y)\big)  \Bigg{]}\label{eq:f-info},
\end{align}
where $p(y|x)$ is the conditional distribution of $y$ on $x$, $q(y)$ is any probability distribution on the random variable $y$, $p(x,y)$ is the joint distribution, and $p(x)$ and $p(y)$ are the resulting marginal distributions. When the subscript $f$ is dropped, $\I$ is simply the Shannon mutual information (i.e., $f(x) = x \log x$ is assumed). The $f$-informativity captures the amount of information obtained about a random variable (e.g., the state $s_t$) by observing another random variable (e.g., sensor observations $o_t$). 

\subsection{Inverting Bounds on the \texorpdfstring{$f$}{Lg}-divergence}
\label{sec:kl inverse}

Let $\mathcal{B}_p$ and $\mathcal{B}_q$ be Bernoulli distributions on $\{0,1\}$ with mean $p$ and $q$ respectively. For $p,q \in [0,1]$, we define:
\begin{equation}
    \mathbb D_{f,\mathcal B}(p \| q) := \mathbb D_f(\mathcal{B}_p \| \mathcal{B}_q) = q f\left( \frac{p}{q}\right) + (1-q) f\left( \frac{1-p}{1-q}\right). \nonumber
\end{equation}

In \revision{Section \ref{sec:one-step}}, we will obtain bounds on \revision{the single-step best achievable expected reward $R_{0\to 1}^\star \in [0,1]$ through bounds that take the form: $\mathbb D_{f,\mathcal B}(R_{0\to 1}^\star \| q) \leq c$ for some $q \in [0,1]$ and an upper bound $c \geq 0$ on $\mathbb D_{f,\mathcal B}(R_{0\to 1}^\star \| q)$}. In order to upper bound \revision{$R_{0\to 1}^\star$}, we will use the \emph{$f$-divergence inverse} (\emph{$f$-inverse} for short):
\begin{equation}
\label{eq:kl inverse}
    \finv(q | c) := \sup \ \{p \in [0,1] \ | \ \mathbb D_{f,\mathcal B}(p \| q) \leq c \}.
\end{equation}
It is then easy to see that $\revision{R_{0\to 1}^\star} \leq \finv(q | c)$. 

Since $\mathbb D_{f,\mathcal B}(\cdot \| \cdot)$ is jointly convex in both arguments \citep{Goldfeld20}, the optimization problem in \eqref{eq:kl inverse} is a convex problem. One can thus compute the $f$-inverse efficiently using a \emph{convex program} \citep[Ch. 4]{Boyd04} with a single decision variable $p$. 
\section{Performance Bound for Single-Step Problems}
\label{sec:one-step}

In this section, we will derive an upper bound on the best achievable reward \revision{$R^\star_{0\to 1}$} in the single time-step decision-making setting. This bound will then be extended to the multi-step setting in Section~\ref{sec:multi-step}. 

When $T=1$, our goal is to upper bound the following quantity: 
\begin{align}
\label{eq:one step objective}
    \revision{R^\star_{0\to 1}}(\sigma_0; r_0) := \ &\underset{\pi_0}{\textrm{sup}} \  \underset{s_0, o_0}{\mathbb{E}}[r_0(s_0, \pi_0(o_0))]  \\
    = \ &\underset{\pi_0}{\textrm{sup}} \  \underset{p_0(s_0)}{\mathbb{E}} \ \underset{\sigma_0(o_0|s_0)}{\mathbb{E}}[r_0(s_0, \pi_0(o_0))].
\end{align}
The notation $\revision{R^\star_{0\to 1}}(\sigma_0; r_0)$ highlights the dependence of the best achievable reward in terms of the robot's sensor and task (as specified by the reward function). As highlighted in Section~\ref{sec:intro}, the amount of information that the robot requires from its sensors in order to obtain high expected reward depends on its task; certain tasks may admit purely open-loop policies that obtain high rewards, while other tasks may require high-precision sensing of the state. We formally define a quantity that captures this intuition and quantifies the \emph{task-relevant information} provided by the robot's sensors. We then demonstrate that this quantity provides an upper bound on $\revision{R^\star_{0\to 1}}(\sigma_0; r_0)$. 

\begin{definition}[Task-relevant information potential]
\label{def:TRI one-step}
Let $\I_f(o_0;s_0)$ be the $f$-informativity between the robot's sensor observation and state. Define:
\begin{equation}
\label{eq:R0 perp}
    \revision{R_{0\to 1}^\perp} := \underset{a_0}{\sup} \  \underset{s_0}{\mathbb{E}}[r_0(s_0, a_0)]
\end{equation}
as the highest achievable reward using an \emph{open-loop} policy. Then define the \emph{task-relevant information potential (TRIP)} of a sensor $\sigma_0$ for a task specified by reward function $r_0$ as:
\begin{equation}
    \tau(\sigma_0; r_0) := \finv(\revision{R_{0\to 1}^\perp} | \I_f(o_0; s_0)).
\end{equation}
\end{definition}

\revision{\begin{remark}
The TRIP depends on the specific choice of $f$ one uses. In Section~\ref{sec:examples}, we will empirically compare the usefulness of different choices of commonly used functions $f$ (Table~\ref{tab:f_func}) from the perspective of establishing fundamental limits.
\end{remark}}

In order to interpret the TRIP, we state two useful properties of the $f$-inverse. 

\begin{restatable}[Monotonicity of $f$-inverse]{proposition}{propmonotone}
\label{prop:f inverse properties}
The $f$-inverse $\finv(q|c)$ is:
\begin{enumerate}
    \item monotonically non-decreasing in $c \geq 0$ for fixed $q \in [0,1]$,
    \item monotonically non-decreasing in $q \in [0,1]$ for fixed $c \geq 0$.
\end{enumerate}
\end{restatable}

\begin{proof}
The first property follows from the fact that increasing $c$ loosens the $f$-divergence constraint in the optimization problem in \eqref{eq:kl inverse}. The proof of the second property 
is provided in Appendix \ref{app:proofs} (\revision{Proposition}\ref{prop:f inverse properties}).
\end{proof}

The TRIP $\tau(\sigma_0; r_0)$ depends on two factors: the $f$-informativity $\I_f(o_0;s_0)$ (which depends on the robot's sensor) and the best reward \revision{$R_{0\to 1}^\perp$} achievable by an open-loop policy (which depends on the robot's task).
Using Proposition~\ref{prop:f inverse properties}, we see that as the sensor provides more information about the state (i.e., as $\I_f(o_0;s_0)$ increases for fixed $\revision{R_{0\to1}^\perp}$), the TRIP is monotonically non-decreasing. Moreover, the TRIP is a monotonically non-decreasing function of $\revision{R_{0\to1}^\perp}$ for fixed $\I_f(o_0;s_0)$. This qualitative dependence is intuitively appealing: if there is a good open-loop policy (i.e., one that achieves high reward), then the robot's sensor can provide a small amount of information about the state and still lead to good overall performance. The specific form of the definition of TRIP is motivated by the result below, which demonstrates that the TRIP upper bounds the best achievable expected reward \revision{$R^\star_{0\to 1}(\sigma_0; r_0)$} in Equation \eqref{eq:one step objective}.

\begin{restatable}[Single-step performance bound]{theorem}{thmonestep}
\label{thm:single-step bound}
The best achievable reward is upper bounded by the task-relevant information potential (TRIP) of a sensor:
\begin{equation}
    \label{eq:one-step bound}
    \small\tau(\sigma_0; r_0) \geq \revision{R^\star_{0\to 1}}(\sigma_0; r_0) = \ \sup_{\pi_0} \  \underset{s_0, o_0}{\mathbb{E}}[r_0(s_0, \pi_0(o_0))].
\end{equation}
\end{restatable}
\begin{proof}
The proof is provided in Appendix \ref{app:proofs} and is inspired by the proof of the generalized Fano inequality presented in Proposition 14 in the work by \cite{Gerchinovitz20}. The bound \eqref{eq:one-step bound} tightens the generalized Fano inequality \citep{Chen16, Gerchinovitz20} by utilizing the $f$-inverse (in contrast to the methods presented by \cite{Chen16} and \cite{Gerchinovitz20}, which may be interpreted as indirectly bounding the $f$-inverse). The result presented here may thus be of independent interest. 
\end{proof}

Theorem \ref{thm:single-step bound} provides a \emph{fundamental bound} on performance (in the sense of Section \ref{sec:intro}) imposed by the sensor for a given single-step task. This bound holds for \emph{any} policy, independent of its complexity or how it is synthesized or learned. \revision{Since the TRIP depends on the choice of $f$-divergence, the bound may be tightened by judiciously choosing $f$; we investigate this empirically in Section~\ref{sec:examples}.}

\section{Performance Bound for Multi-Step Problems: Fano's Inequality with Feedback}
\label{sec:multi-step}

In this section, we derive an upper bound on the best achievable reward $\revision{R_{0\to T}^\star}$ defined in \eqref{eq:control problem} for the general multi time-step setting. The key idea is to extend the single-step bound from Theorem \ref{thm:single-step bound} using a dynamic programming argument. 

Let \revision{$\pi_t^k$}$: \O^{k-t+1} \rightarrow \A$ denote a policy that takes as input the sequence of observations $o_{t:k}$ from time-step $t$ to $k$ (for $k \geq t$). Thus, a policy \revision{$\pi_0^k$} at time-step $k$ utilizes all observations received up to time-step $k$. Given an initial state distribution $p_0$ and an open-loop action sequence $a_{0:t-1}$, define the reward-to-go from time $t \in \{0,\dots,T-1\}$ given  $a_{0:t-1}$ as:
\begin{equation}
\label{eq:Rt}
    \revision{R_{t\to T}} := \underset{\substack{s_{t:T-1}, o_{t:T-1} | \\ a_{0:t-1}}}{\mathbb E} \ \Bigg{[} \sum_{k=t}^{T-1} r_k(s_k, \revision{\pi_t^k}(o_{t:k})) \Bigg{]},
\end{equation}

where the expectation,
\begin{align}
    \underset{\substack{s_{t:T-1}, o_{t:T-1} | \\ a_{0:t-1}}}{\mathbb E} \  \Big{[} \cdot \Big{]} 
\end{align}
is taken with respect to the distribution of states $s_{t:T-1}$ and observations $o_{t:T-1}$ one receives if one propagates $p_0$ using the open-loop sequence of actions from time-steps\footnote{For $t=0$, we use the convention that $a_{0:-1}$ is the empty sequence and:
\begin{align}
    \underset{\substack{s_{0:T-1}, o_{0:T-1} | \\ a_{0:-1}}}{\mathbb E} \  \Big{[} \cdot 
    \Big{]} \ := \underset{s_{0:T-1}, o_{0:T-1}}{\mathbb E} \  \Big{[} \cdot \Big{]}.
\end{align}} 
$0$ to $t-1$, and then applies the closed-loop policies \revision{$\pi_t^t, \pi_t^{t+1}, \dots, \pi_t^{T-1}$} from time-steps $t$ to $T-1$. 
We further define \revision{$R_{T\to T} := 0$}. 

Now, for $t \in \{0,\dots,T-1\}$, define:
\begin{equation}
\label{eq:Rt perp definition}
    \revision{R_{t\to T}^\perp} := \sup_{a_t} \ \Bigg{[} \underset{s_t|a_{0:t-1}}{\mathbb E} \Big{[} r_t(s_t, a_t) \Big{]} + R_{t+1} \Bigg{]},
\end{equation}
and 
\begin{equation}
\label{eq:Rt perp star definition}
    \revision{R_{t \to T}^{\perp \star}} := \sup_{\pi^{t+1}_{t+1}, \dots, \revision{\pi_{t+1}^{T-1}}} \revision{R_{t\to T}^\perp}.
\end{equation}

The following result then leads to a recursive structure for computing an upper bound on $\revision{R_{0\to T}^\star}$.  

\begin{restatable}[Recursive bound]{proposition}{proprecursion}
\label{prop:recursion}
For any $t = 0,\dots,T-1$, the following inequality holds for any open-loop sequence of actions $a_{0:t-1}$:
\begin{small}
\begin{equation}
\label{eq:history-dependent induction}
    \sup_{\pi^t_t, \dots, \revision{\pi_t^{T-1}}} \ \revision{R_{t\to T}} \  \leq \ \underset{\eqqcolon \tau_t(\sigma_{t:T-1}; r_{t:T-1}) }{\underbrace{(T-t) \cdot \finv \Bigg{(} \frac{\revision{R_{t\to T}^{\perp \star}}}{T-t} \ | \ \ \I_f(o_t; s_t) \Bigg{)}}}. 
\end{equation}
\end{small}
\end{restatable}
\begin{proof}
The proof follows a similar structure to Theorem~\ref{thm:single-step bound} and is presented in Appendix \ref{app:proofs}.
\end{proof}

To see how we can use Proposition \ref{prop:recursion}, we first use \eqref{eq:control problem} and \eqref{eq:Rt} to note that the LHS of \eqref{eq:history-dependent induction} for $t=0$ is equal to $\revision{R_{0\to T}^\star}$:
 \begin{equation}
 \label{eq:R0 perp star bound}
    \revision{R_{0\to T}^\star} = \sup_{\pi^0_0, \dots, \revision{\pi_0^{T-1}}} \ \revision{ R_{0\to T}} \ \leq \ \tau_0(\sigma_{0:T-1}; r_{0:T-1}) \ .
\end{equation}

The quantity $\tau_t(\sigma_{t:T-1}; r_{t:T-1})$ may be interpreted as a multi-step version of the TRIP from Definition \ref{def:TRI one-step} \revision{(which again depends on the specific choice of $f$ one uses)}. This quantity depends on the $f$-informativity $\I_f(o_t; s_t)$, which is computed using the distribution $p_t(s_t|a_{0:t-1})$ over $s_t$ that one obtains by propagating $p_0$ using the open-loop sequence of actions $a_{0:t-1}$:
\begin{equation}
\label{eq:multi-step MI}
    \I_f(o_t; s_t) = \inf_q \ \underset{\substack{s_t| \\ a_{0:t-1}}}{\mathbb E} \mathbb D_f\Big{(}\sigma_t(o_t|s_t) \| q(o_t) \Big{)}.
\end{equation}

In addition, $\tau_t(\sigma_{t:T-1}; r_{t:T-1})$ depends on \revision{$ R_{t\to T}^{\perp \star}$}, which is then divided by $(T-t)$ to ensure boundedness between $[0,1]$ (see Assumption \ref{ass:bounded rewards}). The quantity \revision{$ R_{t\to T}^{\perp \star}$} can itself be upper bounded using \eqref{eq:history-dependent induction} with $t+1$, as we demonstrate below. Such an upper bound on \revision{$ R_{t\to T}^{\perp \star}$} for $t=0$ leads to an upper bound on $\revision{R_{0\to T}^\star}$ using \eqref{eq:R0 perp star bound} and the monotonicity of the $f$-inverse (Proposition~\ref{prop:f inverse properties}). Applying this argument recursively leads to Algorithm \ref{a:DP}, which computes an upper bound on $\revision{R_{0\to T}^\star}$. \revision{In Algorithm \ref{a:DP}, we use $\bar R$ to denote recursively-computed upper bounds on the RHS of \eqref{eq:history-dependent induction}.}

\begin{algorithm}[H] 
  \caption{{\small Multi-Step Performance Bound}}
  \label{a:DP}
  \begin{algorithmic}[1]
  \STATE Initialize \revision{$\bar{R}_{T\to T}$}$(a_{0:T-1}) = 0, \ \forall a_{0:T-1}$.
  \FOR{$t = T-1, T-2, \dots, 0$}
  \STATE \revision{$ \ \forall a_{0:t-1}$, compute:} \\ \revision{$\displaystyle\bar{R}_{t\to T}(a_{0:t-1}) := (T-t) \cdot \finv\Big{(} \frac{\bar{R}_{t\to T}^{\perp \star}}{T-t} \Big{|} \I_f(o_t; s_t) \Big{)} ,$\\
  where: \\
  $\bar{R}_{t\to T}^{\perp \star} := \underset{a_t}{\sup} \ \underset{s_t|a_{0:t-1}}{\mathbb E} \Big{[} r_{t}(s_t, a_t) \Big{]} + \bar{R}_{t+1\to T}(a_{0:t}). $} \\
  \ENDFOR
    \RETURN \revision{$\bar{R}_{0\to T}$} (bound on achievable expected reward).
  \end{algorithmic}
\end{algorithm}

\begin{restatable}[Multi-step performance bound]{theorem}{thmmultistep}
\label{thm:multi-step bound}
Algorithm \ref{a:DP} returns an upper bound on the best achievable reward $\revision{R_{0\to T}^\star}$. 
\end{restatable}
\begin{proof}
We provide a sketch of the proof here, which uses backwards induction. In particular, Proposition \ref{prop:recursion} leads to the inductive step. See Appendix \ref{app:proofs} for the complete proof. 

We prove that for all $t=T-1,\dots,0$, 
\begin{equation}
\label{eq:algorithm bound sketch}
    \sup_{\pi^t_t, \dots, \revision{\pi_t^{T-1}}} \ \revision{R_{t\to T}} \leq \revision{\bar{R}_{t\to T}}(a_{0:t-1}), \ \forall a_{0:t-1}.
\end{equation}
Thus, in particular, 
\begin{equation}
    \revision{R_{0\to T}^\star} = \sup_{\pi^0_0, \dots, \revision{\pi_0^{T-1}}} \revision{\ R_{0\to T} \leq \bar{R}_{0\to T}}.
\end{equation}

We prove \eqref{eq:algorithm bound sketch} by backwards induction starting from $t=T-1$. We first prove the base step. Using \eqref{eq:history-dependent induction}, we obtain:
\begin{equation}
\label{eq:induction base sketch}
        \small\sup_{\pi^{T-1}_{T-1}} \ \revision{R_{T-1\to T}} \  \leq \ \finv \Big{(} \revision{R_{T-1\to T}^{\perp \star}} \ | \ \ \I_f(o_{T-1}; s_{T-1}) \Big{)}.
\end{equation}
Using the fact that \revision{$R_{T\to T} = 0$}, we can show that \revision{$R_{T-1\to T}^{\perp \star} = \bar{R}_{T-1\to T}^{\perp \star}$}. 
Combining this with \eqref{eq:induction base sketch} and the monotonicity of the $f$-inverse (Proposition \ref{prop:f inverse properties}), we see:
\begin{small}
    \begin{align}
    \sup_{\pi^{T-1}_{T-1}} \ \revision{R_{T-1\to T}} \  &\leq \ \finv \Big{(} \revision{\bar{R}_{T-1\to T}^{\perp \star}} \ | \ \ \I_f(o_{T-1}; s_{T-1}) \Big{)} \\
    &= \revision{\bar{R}_{T-1\to T}}(a_{0:T-2}).
\end{align}
\end{small}

In order to prove the induction step, suppose that for $t \in \{0,\dots,T-2\}$, we have
\begin{equation}
\label{eq:induction hypothesis sketch}
    \sup_{\pi^{t+1}_{t+1}, \dots, \revision{\pi_{t+1}^{T-1}}} \ \revision{R_{t+1\to T}} \leq \revision{\bar{R}_{t+1\to T}}(a_{0:t}).
\end{equation}
We then need to show that
\begin{equation}
\label{eq:induction WTS sketch}
    \sup_{\pi^{t}_{t}, \dots, \revision{\pi_{t}^{T-1}}} \ \revision{R_{t\to T}} \leq \revision{\bar{R}_{t\to T}}(a_{0:t-1}).
\end{equation}
We can use the induction hypothesis \eqref{eq:induction hypothesis sketch} to show that \revision{$R_{t\to T}^{\perp \star} \leq \bar{R}_{t\to T}^{\perp \star}$}. Combining this with \eqref{eq:history-dependent induction} and the monotonicity of the $f$-inverse (Proposition \ref{prop:f inverse properties}), we obtain the desired result \eqref{eq:induction WTS sketch}:
\begin{small}
\begin{align}
  \sup_{\pi^t_t, \dots, \revision{\pi_t^{T-1}}} \ \revision{R_{t\to T}} \  &\leq \ \textstyle(T-t) \cdot \finv \Bigg{(} \frac{\revision{\bar{R}_{t\to T}^{\perp \star}}}{T-t} \ | \ \ \I_f(o_t; s_t) \Bigg{)} \\
  &= \revision{\bar{R}_{t\to T}}(a_{0:t-1}).
\end{align}
\end{small}
\end{proof}

\section{Numerical Implementation}
\label{sec:implementation}

In order to compute the single-step bound using Theorem~\ref{thm:single-step bound} or the multi-step bound using Algorithm \ref{a:DP}, we require the ability to compute (or bound) three quantities: (i) the $f$-inverse, (ii) the $f$-informativity $\I_f(o_t; s_t)$, and (iii) the quantity \revision{$\bar{R}_{t\to T}^{\perp \star}$}. As described in Section \ref{sec:kl inverse}, we can compute the $f$-inverse efficiently using a convex program \citep[Ch. 4]{Boyd04} with a single decision variable. There are a multitude of solvers for convex programs including Mosek \citep{MOSEK} and the open-source solver SCS ~\citep{SCS}. Next, we describe the computation of $\I_f(o_t; s_t)$ and \revision{$\bar{R}_{t\to T}^{\perp \star}$} in different settings.

\subsection{Analytic Computation}
\label{sec: analytic computation}

In certain settings, one can compute $\I_f(o_t; s_t)$ and \revision{$\bar{R}_{t\to T}^{\perp \star}$} exactly. We discuss two such settings of interest below. 

{\bf Discrete POMDPs.} In cases where the state space $\S$, action space $\A$, and observation space $\O$ are finite, one can compute $\I_f(o_t; s_t)$ exactly by propagating the initial state distribution $p_0$ forward using open-loop action sequences $a_{0:t-1}$ and using the expression \eqref{eq:multi-step MI} (which can be evaluated exactly since we have discrete probability distributions). The expectation term in \revision{$\bar{R}_{t\to T}^{\perp \star}$} can be computed similarly. In addition, the supremum over actions can be evaluated exactly via enumeration. 

{\bf Linear-Gaussian systems with finite action spaces.} One can also perform exact computations in cases where (i) the state space $\S$ is continuous and the dynamics $p_t(s_{t}|s_{t-1},a_{t-1})$ are given by a linear dynamical system with additive Gaussian uncertainty, (ii) the observation space $\O$ is continuous and the sensor model $\sigma_t(o_t|s_t)$ is such that the observations are linear (in the state) with additive Gaussian uncertainty, (iii) the initial state distribution $p_0$ is Gaussian, and (iv) the action space $\A$ is finite. In such settings, one can analytically propagate $p_0$ forward through open-loop action sequences $a_{0:t-1}$ using the fact that Gaussian distributions are preserved when propagated through linear-Gaussian systems (similar to Kalman filtering \citep{Thrun05}). One can then compute $\I_f(o_t; s_t)$ using \eqref{eq:multi-step MI} by leveraging the fact that all the distributions involved are Gaussian, for which KL divergences (and some $f$-divergences) can be computed in closed form \citep{Duchi16}. One can also compute \revision{$\bar{R}_{t\to T}^{\perp \star}$} exactly for any reward function that permits the analytic computation of the expectation term using a Gaussian (e.g., quadratic reward functions); the supremum over actions can be evaluated exactly since $\A$ is finite. 

\subsection{Computation via Sampling and Concentration Inequalities}
\label{sec:sampling}

{\bf General settings.} Next, we consider more general settings with: (i) continuous state and observation spaces, (ii) arbitrary (e.g, non-Gaussian/nonlinear) dynamics $p_t(s_{t}|s_{t-1},a_{t-1})$, which are potentially not known analytically, but can be sampled from (e.g., as in a simulator), (iii) arbitrary (e.g., non-Gaussian/nonlinear) sensor $\sigma_t(o_t|s_t)$, but with a probability density function that can be numerically evaluated given any particular state-observation pair, (iv) an arbitrary initial state distribution $p_0$ that can be sampled from, and (v) a finite action space. Our bound is thus broadly applicable, with the primary restriction being the finiteness of $\A$; we leave extensions to continuous action spaces for future work (see Section \ref{sec:conclusion}).

We first discuss the computation of \revision{$\bar{R}_{t\to T}^{\perp \star}$}. Since the supremization over actions can be performed exactly (due to the finiteness of $\A$), the primary challenge here is to evaluate the expectation:
\begin{equation}
\label{eq:expectation in Rt}
    \underset{s_t|a_{0:t-1}}{\mathbb E} \Big{[} r_{t}(s_t, a_t) \Big{]}.
\end{equation}
We note that any upper bound on this expectation leads to an upper bound on \revision{$\bar{R}_{t\to T}^{\perp \star}$}, and thus a valid upper bound on \revision{$R_{0\to T}^\star$} (due to the monotonicity of the $f$-inverse; Proposition~\ref{prop:f inverse properties}). One can thus obtain a high-confidence upper bound on \eqref{eq:expectation in Rt} by sampling states $s_t|a_{0:t-1}$ and using any concentration inequality \citep{Boucheron13}. In particular, since we assume boundedness of rewards (Assumption \ref{ass:bounded rewards}), we can use Hoeffding's inequality. 

\begin{theorem}[Hoeffding's inequality \citep{Boucheron13}]
\label{thm:hoeffding}
Let $z$ be a random variable bounded within $[0,1]$, and let $z_1,\dots, z_n$ denote i.i.d. samples. \revision{Then, with probability at least $1 - \delta$ (over the sampling of $z_1,\dots,z_n$), the following bound holds}:
\begin{equation}
    \mathbb{E}[z] \leq \frac{1}{n}\sum_{i=1}^n z_i + \sqrt{\frac{\log(1/\delta)}{2n}}.
\end{equation}
\end{theorem}
In our numerical examples (Section \ref{sec:examples}), we utilize a slightly tighter version of Hoeffding's inequality (see Appendix \ref{app:hoeffding}). 

Next, we discuss the computation of $\I_f(o_t; s_t)$. Again, we note that any upper bound on $\I_f(o_t; s_t)$ yields a valid upper bound on \revision{$R_{0\to T}^\star$} due to the monotonicity of the $f$-inverse. In general, since $f$-informativity is the infimum among all distributions $q(o_t)$ over the observations at time $t$, the marginal distribution $\sigma_t(o_t)$ provides a valid upper bound:
\begin{small}
\begin{equation}
    \resizebox{.85\hsize}{!}{$\I_f(o_t; s_t) \leq \mathbb D_f\Big{(}p_t(s_t|a_{0:t-1})\sigma_t(o_t|s_t) \| p_t(s_t|a_{0:t-1})\sigma_t(o_t) \Big{)}.$}
\end{equation}
\end{small}
For KL divergence and mutual information specifically, we utilize \emph{variational bounds}; in particular, we use the ``leave-one-out bound" \citep{Poole19}:
\begin{small}
\begin{equation}
\label{eq:variational MI bound}
    \I(o_t; s_t)  \leq \mathbb{E} \Bigg{[} \frac{1}{K} \sum_{i=1}^K \Bigg{[} \log \frac{\sigma_t(o_t^{[i]}|s_t^{[i]})}{\frac{1}{K-1} \sum_{j\neq i} \sigma_t(o_t^{[i]}|s_t^{[j]})} \Bigg{]} \Bigg{]},
\end{equation}
\end{small}

where the expectation is over size-$K$ batches $\{ (s_t^{[i]}, o_t^{[i]}) \}_{i=1}^K$ of sampled states $s_t|a_{0:t-1}$ and observations sampled using $\sigma_t(o_t|s_t)$. The quantity $\sigma_t(o_t^{[i]}|s_t^{[i]})$ denotes (with slight abuse of notation) the evaluation of the density function corresponding to the sensor model. Since the bound \eqref{eq:variational MI bound} is in terms of an expectation, one can again obtain a high-confidence upper bound by sampling state-observation batches and applying a concentration inequality (e.g., Hoeffding's inequality if the quantity inside the expectation is bounded). 

We note that the overall implementation of Algorithm~\ref{a:DP} may involve the application of multiple concentration inequalities (each of which holds with some confidence $1-\delta_i$). One can obtain the overall confidence of the upper bound on \revision{$R_{0\to T}^\star$} by using a union bound: $1 - \delta = 1 - \sum_i \delta_i$.

\subsection{Tightening the Bound with Multiple Horizons}

We end this section by discussing a final implementation detail. Let $T$ denote the time horizon of interest (as in Section~\ref{sec:problem formulation}). For any $H \in \{1,\dots,T\}$, one can define
\begin{equation}
    \revision{R^\star_{0\to H}} := \underset{\pi_{0:H-1}}{\sup} \ \ \underset{\substack{s_{0:H-1} \\ o_{0:H-1}}}{\mathbb{E}}\Bigg{[}\sum_{t=0}^{H-1} r_t(s_t, \pi_t(o_{0:t}))\Bigg{]}
\end{equation}
as the best achievable reward for a problem with horizon $H$ (instead of $T$). One can then apply Algorithm \ref{a:DP} to compute an upper bound on \revision{$R^\star_{0\to H}$}. Since rewards are assumed to be bounded within $[0,1]$ (Assumption \ref{ass:bounded rewards}), we can observe that \revision{$R_{0\to T}^\star\leq R^\star_{0\to H} + (T-H)$}; \revision{ this is equivalent to computing the bound using a horizon $H$ and then adding a reward of $1$ for times beyond this horizon}. In practice, we sometimes find that this bound provides a tighter bound on \revision{$R_{0\to T}^\star$} for some $H < T$ (as compared to directly applying Algorithm \ref{a:DP} with a horizon of $T$). For our numerical examples, we thus sweep through different values for the horizon $H$ and report the lowest upper bound $\revision{R^\star_{0\to H}} + (T-H)$.

\section{Examples}
\label{sec:examples}

We demonstrate our approach on three examples: (i) the lava problem from the POMDP literature, (ii) an example with continuous state and observation spaces corresponding to a robot catching a freely-falling object, and (iii) obstacle avoidance using a depth sensor with non-Gaussian noise. We illustrate the strength of our upper bounds on these examples by comparing them against the performance achieved by concrete control policies (i.e., lower bounds on achievable performance). 
We also demonstrate the applicability of our approach for establishing the superiority of one sensor over another (from the perspective of a given task). Code for all examples can be found at: \href{https://github.com/irom-lab/performance-limits}{https://github.com/irom-lab/performance-limits}. 

\subsection{Lava Problem}
\label{sec:lava problem}

The first example we consider is the lava problem (Figure~\ref{fig:lava}) \citep{Cassandra94, Florence17, Pacelli20} from the POMDP literature. 

{\bf Dynamics.}  The setting consists of five discrete states (Figure~\ref{fig:lava}) and two actions (\texttt{left} and \texttt{right}). If the robot falls into the lava state, it remains there (i.e., the lava state is absorbing). If the robot attempts to go left from state 1, it remains at state 1. The initial state distribution $p_0$ is chosen to be uniform over the non-lava states. 

{\bf Sensor.}
The robot is equipped with a sensor that provides a noisy state estimate. The sensor reports the correct state (i.e., $o_t=s_t$) with probability $p_\text{correct}$, \revision{and a uniformly randomly chosen incorrect state with probability $1-p_\text{correct}$.}

{\bf Rewards.}
The robot's objective is to navigate to the goal state (which is an absorbing state), within a time horizon of $T=5$. This objective is encoded via a reward function $r_t(s_t, a_t)$, which is purely state-dependent. The reward associated with being in the lava is 0; the reward associated with being at the goal is 1; the reward at all other states is $0.1$. 
 
{\bf Results.} An interesting feature of this problem is that it admits a purely \emph{open-loop} policy that achieves a high expected reward. In particular, consider the following sequence of actions: \texttt{left}, \texttt{right}, \texttt{right}. No matter which state the robot starts from, this sequence of actions will steer the robot to the goal state (recall that the goal is absorbing). Given the initial distribution and rewards above, this open-loop policy achieves an expected reward of 3.425. Suppose we set $p_\text{correct} = 1/5$ (i.e., the sensor just uniformly randomly returns a state estimate and thus provides no information regarding the state). In this case, Algorithm~\ref{a:DP} returns an upper bound: $3.5 \geq \revision{R_{0\to T}^\star}$. 

Next, we plot the upper bounds provided by Algorithm \ref{a:DP} for different values of sensor noise by varying $p_\text{correct}$. \revision{With the freedom to choose the exact $f$-divergence to use in Algorithm~\ref{a:DP}, we evaluate a series of common $f$-divergences shown in Table \ref{tab:f_func}.}

\begin{table}[H]
    \centering
    \small\centering
    \caption{$f$-Divergences Used to Compute Upper Bounds.}
    \begin{tabular}{c c}
    \toprule
        Divergence &  Corresponding $f$\\
    \midrule
        Kullback-Leibler & $x\log(x)$\\
        Negative log & $-\log(x)$\\
        Total variation & $\frac{1}{2}|x-1|$\\
        Pearson $\chi^2$ & $(x-1)^2$\\
        Jensen-Shannon & $-(x+1)\ln(\frac{x+1}{2})+x\ln x$\\
        Squared Hellinger & $(\sqrt{t}-1)^2$\\
        Neyman $\chi^2$ & $\frac{1}{t}-1$\\
    \bottomrule
    \end{tabular}
    \label{tab:f_func}
\end{table}

Three of the resulting bounds are shown in Figure \ref{fig:lava1 results}. The three bounds are chosen to be plotted because: 1) KL divergence is one of the best-known divergences, 2) Total variation distance provides the tightest bounds among those computed when sensor accuracy is low, and 3) Neyman $\chi^2$-divergence provides the tightest bounds among computed when sensor accuracy is higher. Since the results for different sensor noise levels are independent of each other, we can always choose the particular $f$-divergence that returns the tightest bound. 

\begin{figure}[t]
\begin{center}
\includegraphics[width=0.99\columnwidth]{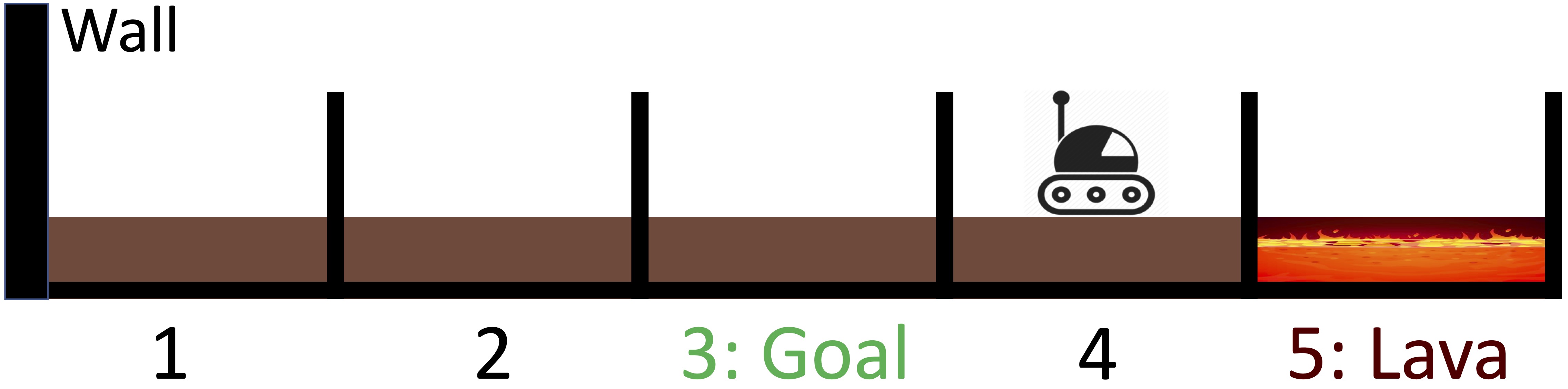}
\end{center} 
\vspace{-5pt}
\caption{ An illustration of the lava problem. The
robot needs to navigate to a goal without falling into the lava (using a noisy sensor). \label{fig:lava}}
\end{figure}

\begin{figure}[t]
\begin{center}
\includegraphics[width=0.99\columnwidth]{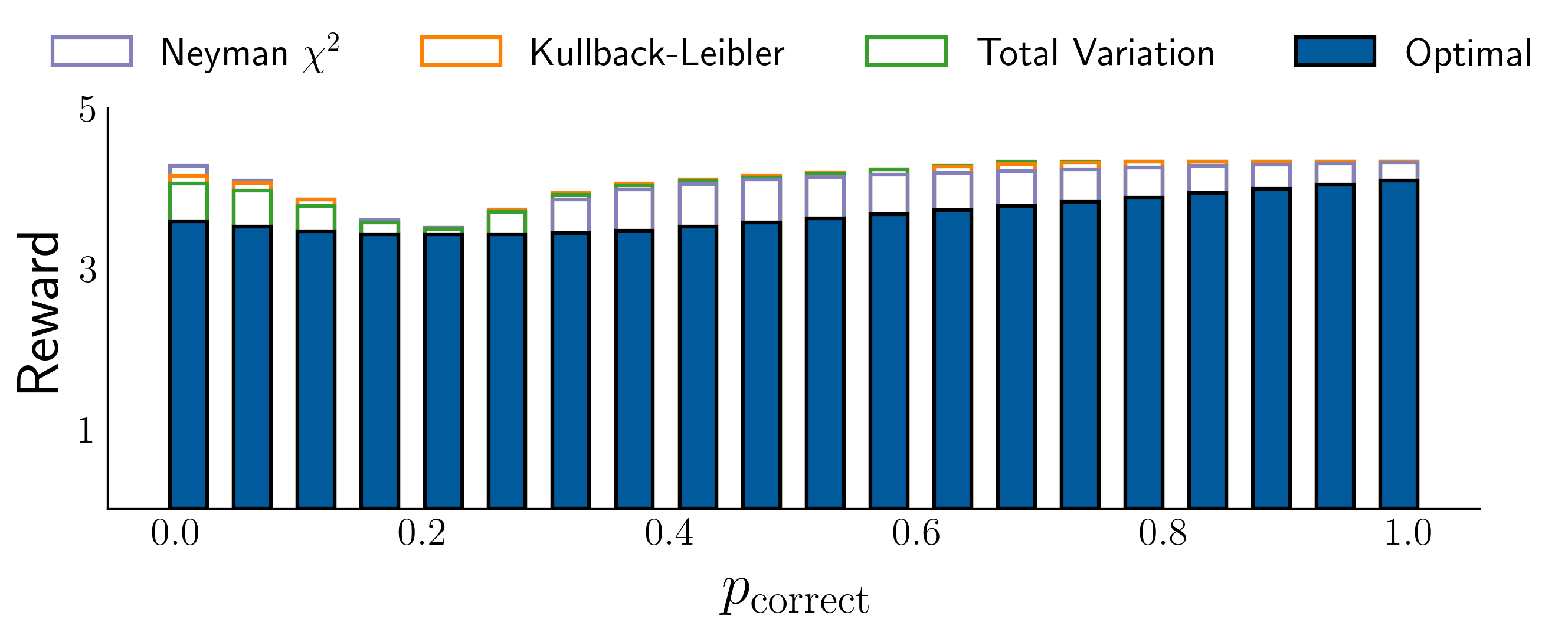}
\end{center} 
\vspace{-15pt}
\caption{Results for the lava problem. We compare the upper bounds on achievable expected rewards computed by our approach using three different $f$-divergences with the optimal POMDP solution for different values of sensor noise. \label{fig:lava1 results}}
\end{figure}

Since the lava problem is a finite POMDP, one can compute $\revision{R_{0\to T}^\star}$ \emph{exactly} using a POMDP solver. Figure~\ref{fig:lava1 results} compares the upper bounds on $\revision{R_{0\to T}^\star}$ returned by Algorithm~\ref{a:DP} with $\revision{R_{0\to T}^\star}$ computed using the \texttt{pomdp\_py} package~\citep{Zheng20}. The figure illustrates that our approach provides strong bounds on the best achievable reward for this problem. We also note that computation of the POMDP solution (for each value of $p_\text{correct}$) takes $\sim$20s, while the computation of the bound takes $\sim$0.2s (i.e., $\sim100\times$ faster).

{\bf Tightening the upper bound.} In Figure \ref{fig:lava1 results}, we note that different $f$-divergences yield different upper bound values. While all bounds are valid, we are motivated to search for the tightest bound by optimizing the function $f$. To do so, we define \revision{a family of }generic piecewise-linear functions that are convex and satisfy $f(1)=0$; then for a given sensor noise level, we \revision{minimize the} upper bounds over the \revision{defined family of }piecewise-linear functions \revision{using \texttt{scipy.optimize}}. To define such a family of functions, we divide the interval $(0,2]$ into \revision{$n$} steps of equal length, each step with a corresponding slope $s_i$. Since $f$ is a function on $\mathbb R^+$, we extend the last step interval to positive infinity. The \revision{$n$} slopes, along with the constraints that $f$ is convex (therefore continuous) and $f(1) = 0$, complete the definition of $f$. 
The results obtained by setting \revision{$n=10$} are shown in Figure \ref{fig:lava PL results}, and the resulting functions $f$ that give the tightest bounds for \revision{the first 15} sensor noise level are shown in Figure~\ref{fig:lava PL funcs}. We can see that the functions tightening the bounds follow a general trend from more parabolic to more log-like as $p_\text{correct}$ increases.

\begin{figure}[t]
    \begin{center}
        \includegraphics[width=0.99\columnwidth]{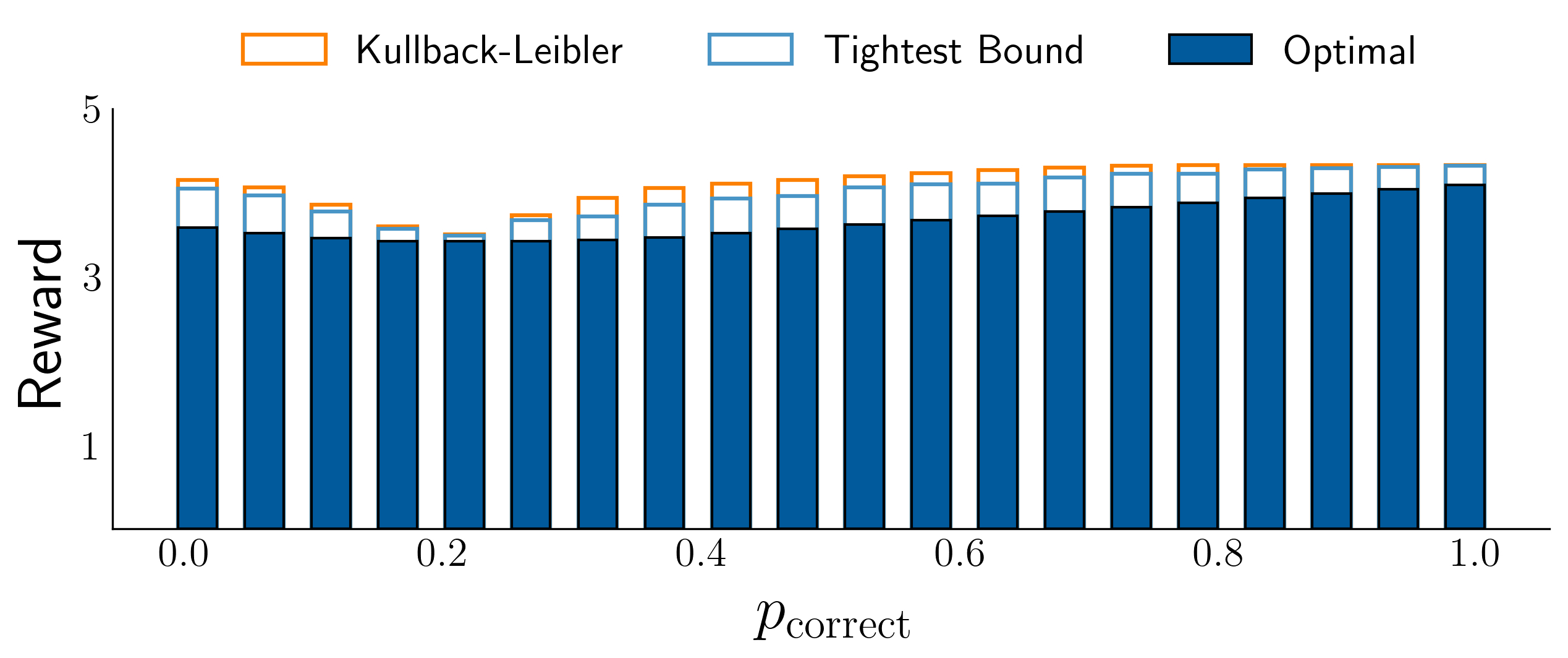}
    \end{center}
    \vspace{-15pt}
    \caption{Tightest upper bounds for the lava problem. We compare the optimized upper bounds on achievable expected rewards with the bounds computed by Algorithm \ref{a:DP} using KL divergence and the optimal POMDP solution.}
    \label{fig:lava PL results}
\end{figure}

\begin{figure}[t]
    \begin{center}
        \includegraphics[width=0.99\columnwidth]{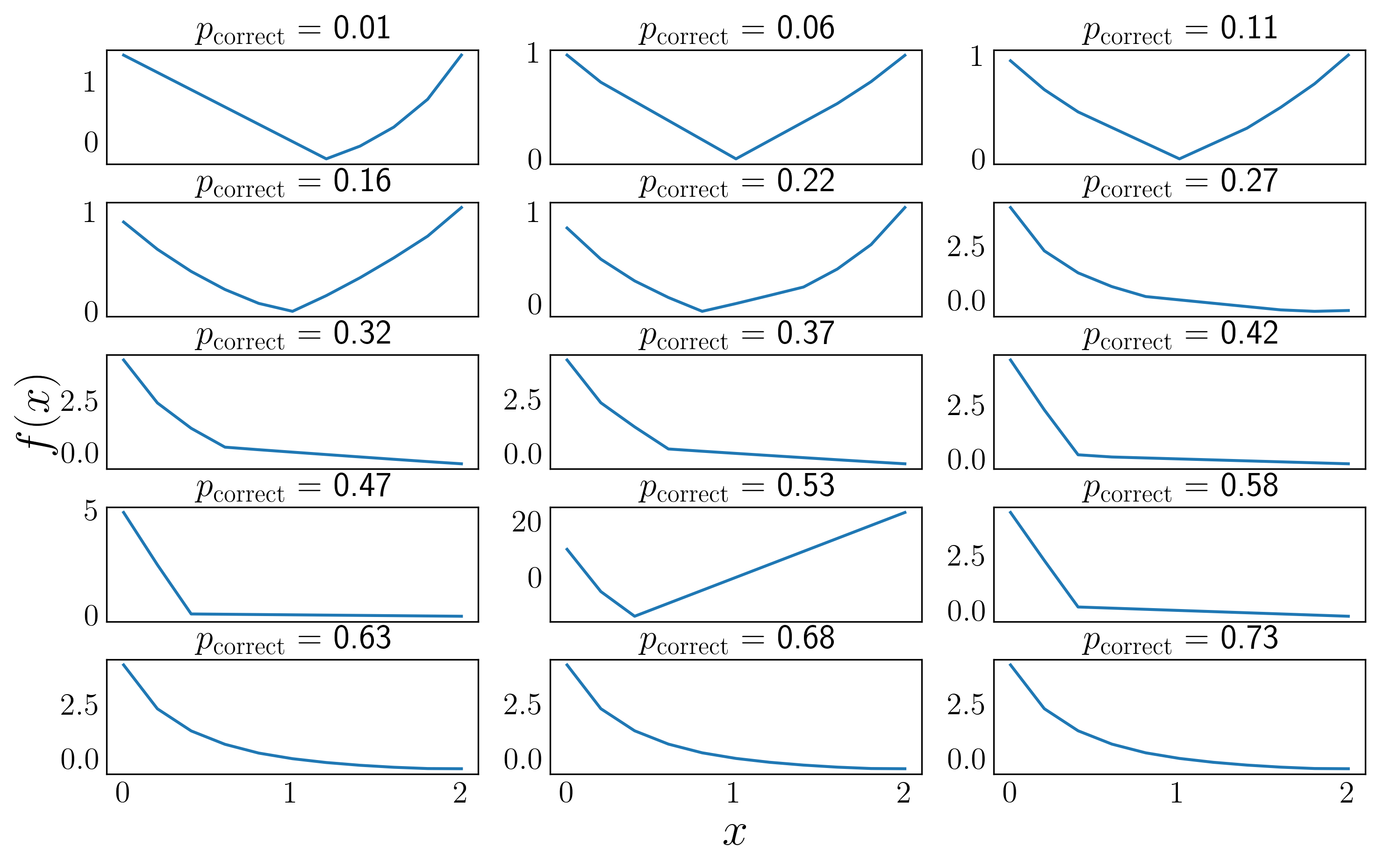}
    \end{center}
    \vspace{-15pt}
    \caption{Functions $f$ that provide the tightest bounds for corresponding sensor noise level.}
    \label{fig:lava PL funcs}
\end{figure}

{\bf Comparing with generalized Fano's inequality.} In the work by \cite{Chen16}, the generalized Fano's inequality gives a lower bound for the Bayesian risk, which translates to an upper bound for the expected reward of:
\begin{equation}
\revision{R^\star_{0\to 1}}(\sigma_0, r_0) \leq\frac{\mathbb I(o_0;s_0)+\log(1+\revision{R_{0\to 1}^\perp})}{\log(1/(1-\revision{R_{0\to 1}^\perp})}.
\label{eq:fano}
\end{equation}

We compare the upper bounds obtained via our approach (Theorem~\ref{thm:single-step bound}) using KL divergence with the ones obtained by equation (\ref{eq:fano}) for the single-step lava problem. The results shown in Figure \ref{fig:fano} demonstrate that our novel bound in Theorem~\ref{thm:single-step bound} is significantly tighter, especially for larger values of $p_\text{correct}$. \revision{Indeed, the bound from Theorem~\ref{thm:single-step bound} is \emph{always} at least as tight as inequality \eqref{eq:fano} (not just for this specific example). This is because the RHS of \eqref{eq:one-step bound} is a lower bound on the RHS of \eqref{eq:fano}}. Our bound may thus be of independent interest for applications considered by \cite{Chen16} \revision{such as} establishing sample complexity bounds for various learning problems.

\begin{figure}[t]
    \begin{center}
        \includegraphics[width=0.99\columnwidth]{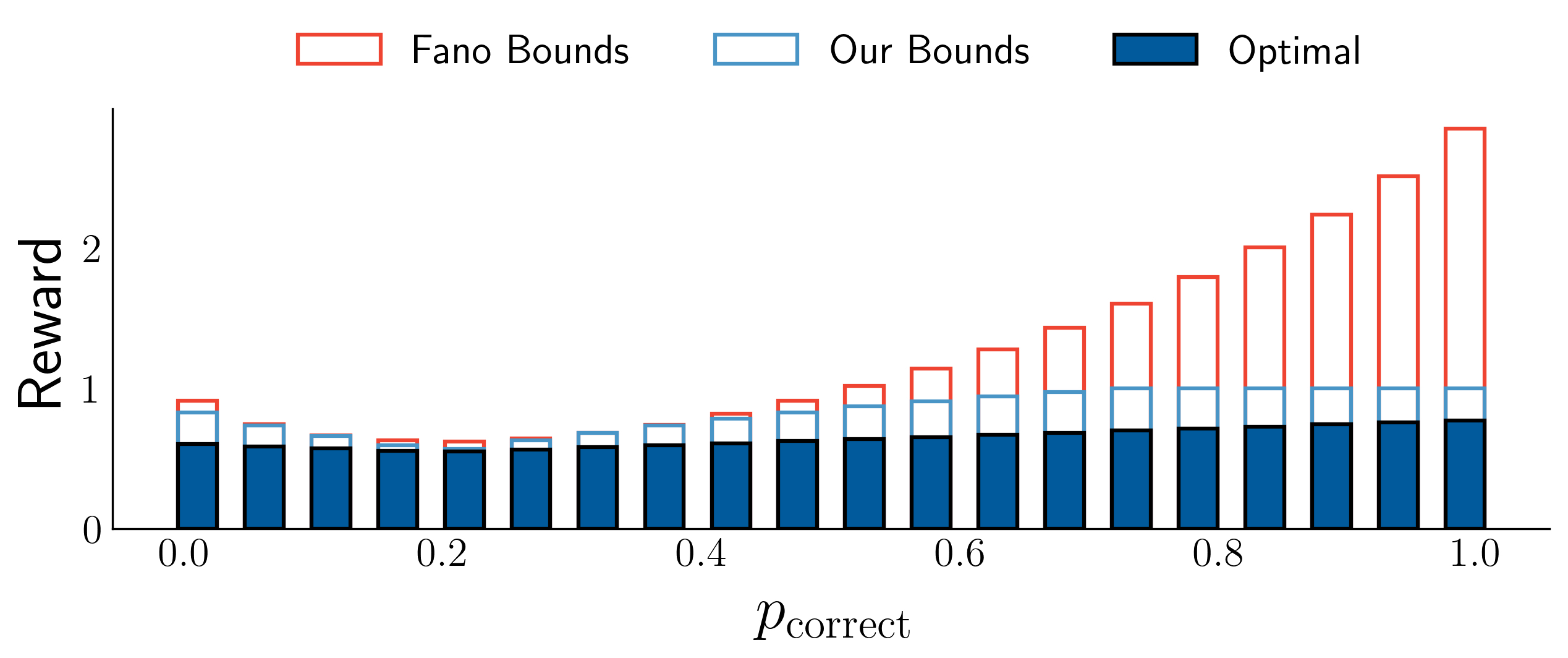}
    \end{center}
    \vspace{-15pt}
    \caption{Results for the one-step lava problem. We compare the upper bounds on achievable expected rewards computed via our approach (using KL divergence) with the bounds computed via generalized Fano's inequality.}
    \label{fig:fano}
\end{figure}


\subsection{Catching a Falling Object}

Next, we consider a problem with continuous state and observation spaces. The goal of the robot is to catch a freely-falling object such as a ball (Figure \ref{fig:ball catching}).

{\bf Dynamics.} We describe the four-dimensional state of the robot-ball system by $s_t := [x_t^\text{rel}, y_t^\text{rel}, v_t^\text{x,ball}, v_t^\text{y,ball}] \in \mathbb{R}^4$, where $[x_t^\text{rel}, y_t^\text{rel}]$ is the relative position of the ball with respect to the robot, and $[v_t^\text{x,ball}, v_t^\text{y,ball}]$ corresponds to the ball's velocity. The action $a_t$ is the horizontal speed of the robot and can be chosen within the range $[-0.4,0.4]$m/s (discretized in increments of 0.1m/s).
The dynamics of the system are given by: 
\begin{equation}
\label{eq:catching dynamics}
s_{t+1} = 
\begin{bmatrix}
\vspace{5pt}
x_{t+1}^\text{rel} \\
\vspace{5pt}
y_{t+1}^\text{rel} \\
\vspace{5pt}
v_{t+1}^\text{x,ball} \\
\vspace{5pt}
v_{t+1}^\text{y,ball}
\end{bmatrix}
=
\begin{bmatrix}
\vspace{5pt}
x_{t}^\text{rel} + (v_{t}^\text{x,ball} - a_t) \Delta t  \\
\vspace{5pt}
y_{t}^\text{rel} + \Delta t v_{t}^\text{y,ball} \\
\vspace{5pt}
v_{t}^\text{x,ball} \\
\vspace{5pt}
v_{t}^\text{y,ball} - g \Delta t
\end{bmatrix},
\end{equation}
where $\Delta t = 1$ is the time-step and $g=0.1$m/s$^2$ is chosen such that the ball reaches the ground within a time horizon of $T=5$. The initial state distribution $p_0$ is chosen to be a Gaussian with mean $[0.0 \text{m},1.05\text{m},0.0\text{m/s},0.05\text{m/s}]$ and diagonal covariance matrix $\text{diag}([0.01^2, 0.1^2, 0.2^2, 0.1^2])$. 

{\bf Sensor.} The robot's sensor provides a noisy state estimate $o_t = s_t + \epsilon_t$, where $\epsilon_t$ is drawn from a Gaussian distribution with zero mean and diagonal covariance matrix $\eta \cdot \text{diag}([0.5^2, 1.0^2, 0.75^2, 1.0^2])$. Here, $\eta$ is a noise scale that we will vary in our experiments. 

{\bf Rewards.} The reward function $r_t(s_t,a_t)$ is chosen to encourage the robot to track the ball's motion. In particular, we choose $r_t(s_t,a_t) = \max(1 - 2|x_t^\text{rel}|, 0)$. The reward is large when $x_\text{rel}$ is close to 0 (with a maximum reward of 1 when $x_t^\text{rel} = 0$); the robot receives no reward if $|x_t^\text{rel}| \geq 0.5$. The robot's goal is to maximize the expected cumulative reward over a time horizon of $T=5$. 

 \begin{figure}[t]
\begin{center}
\includegraphics[width=0.99\columnwidth]{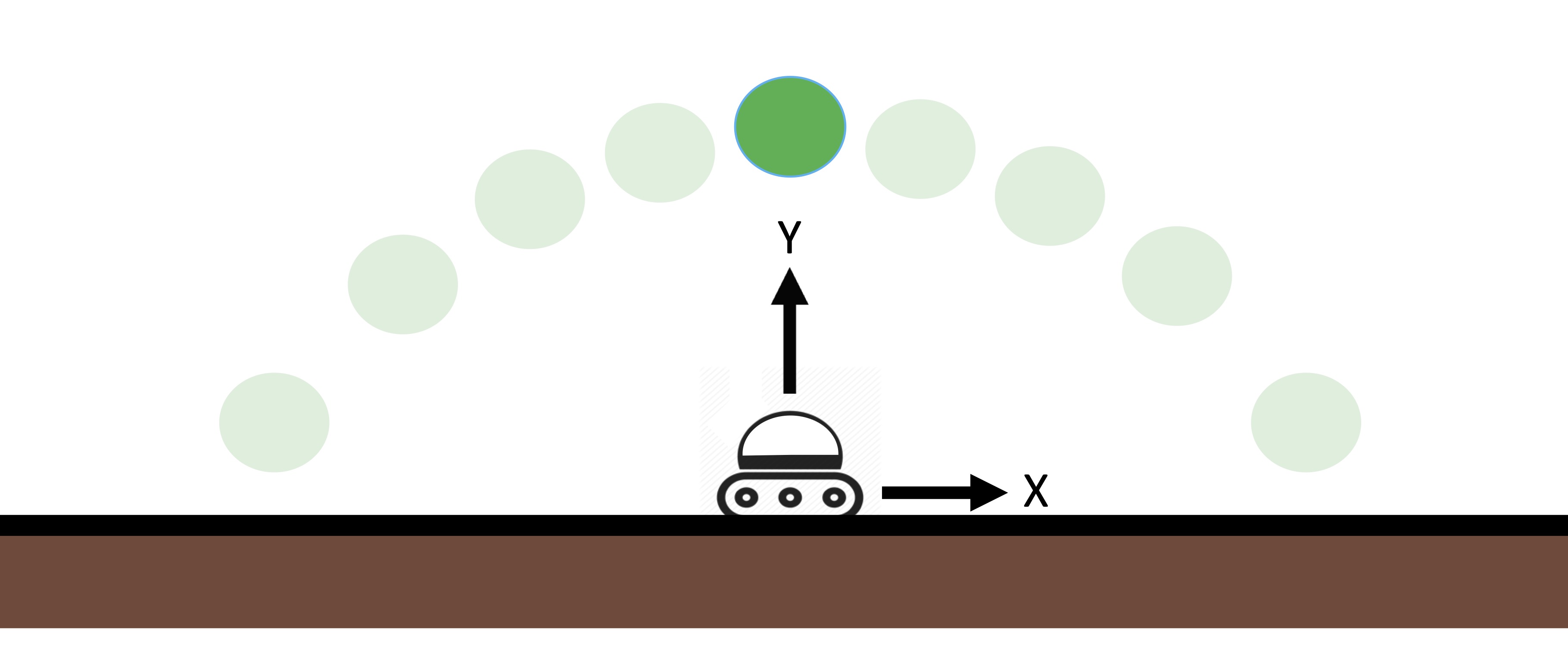}
\end{center} 
\vspace{-5pt}
\caption{An illustration of the ball-catching example with continuous state and observation spaces. The robot is constrained to move horizontally along the ground and can control its speed. Its goal is to track the position of the falling ball using a noisy estimate of the ball's state. \label{fig:ball catching}}
\end{figure}

\begin{figure}[t]
\begin{center}
\includegraphics[width=0.99\columnwidth]{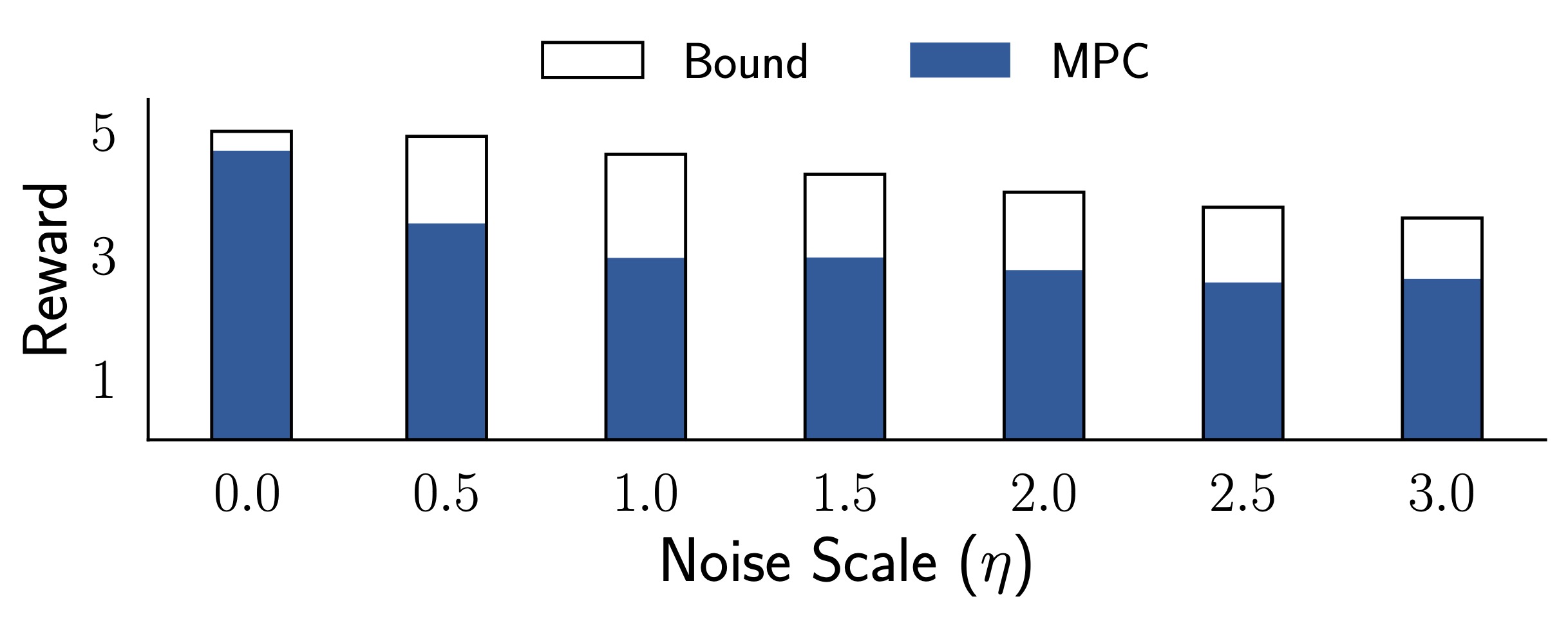}
\end{center} 
\vspace{-10pt}
\caption{Results for the ball-catching example. We compare the upper bounds on achievable expected rewards with the expected rewards using MPC combined with Kalman filtering for different values of sensor noise (results for MPC are averaged across five evaluation seeds; the std. dev. across seeds is too small to visualize).  \label{fig:ball catching results}}
\vspace{-12pt}
\end{figure}

{\bf Results.} Unlike the lava problem, this problem does not admit a good open-loop policy since the initial distribution on $v_{0}^\text{x,ball}$ is symmetric about 0; thus, the robot does not have \emph{a priori} knowledge of the ball's x-velocity direction (as illustrated in Figure \ref{fig:ball catching}). In this example, we choose KL divergence as the specific $f$-divergence to compute bounds with. Figure \ref{fig:ball catching results} plots the upper bound on the expected cumulative reward obtained using Algorithm~\ref{a:DP} for different settings of the observation noise scale $\eta$. Since the dynamics \eqref{eq:catching dynamics} are affine, the sensor model is Gaussian, and the initial state distribution is also Gaussian, we can apply the techniques described in Section \ref{sec: analytic computation} for \emph{analytically} computing the quantities of interest in Algorithm~\ref{a:DP}. 

Figure \ref{fig:ball catching results} also compares the upper bounds on \revision{the highest achievable expected reward $\revision{R_{0\to T}^\star}$ with \emph{lower bounds} on this quantity. To do this, we note that the expected reward achieved by any \emph{particular}  control policy provides a lower bound on $\revision{R_{0\to T}^\star}$. In this example, we compute lower bounds by computing the expected reward achieved by} a model-predictive control (MPC) scheme combined with a Kalman filter for state estimation. We estimate this expected reward using 100 initial conditions sampled from $p_0$. Figure \ref{fig:ball catching results} plots the average of these expected rewards across five random seeds (the resulting standard deviation is too small to visualize). As the figure illustrates, the MPC controller obeys the fundamental bound on reward computed by our approach. Moreover, the performance of the controller qualitatively tracks the degradation of achievable performance predicted by the bound as $\eta$ is increased. Finally, we observe that sensors with noise scales $\eta = 1$ and higher are \emph{fundamentally limited} as compared to a noiseless sensor. This is demonstrated by the fact that the MPC controller for $\eta = 0$ achieves higher performance than the fundamental limit on performance for $\eta = 1$.

\subsection{Obstacle Avoidance with a Depth Sensor}

For our final example, we consider the problem of obstacle avoidance using a depth sensor (Figure \ref{fig:anchor}). This is a more challenging problem with higher-dimensional (continuous) state and observation spaces, and non-Gaussian sensor models. 

{\bf State and action spaces.} The robot is initialized at the origin with six cylindrical obstacles of fixed radius placed randomly in front of it. The state $s_t \in \mathbb{R}^{12}$ of this system describes the locations of these obstacles in the environment. In addition, we also place ``walls" enclosing a workspace $[-1,1]$m $\times$ $[-0.1,1.2]$m (these are not visualized in the figure to avoid clutter). The initial state distribution $p_0$ corresponds to uniformly randomly choosing the x-y locations of the cylindrical obstacles from the set $[-1,1]$m $\times$ $[0.9,1.1]$m. The robot's goal is to navigate to the end of the workspace by choosing a motion primitive to execute (based on a noisy depth sensor described below). Figure \ref{fig:anchor} illustrates the set of ten motion primitives the robot can choose from; this set corresponds to the action space. 

{\bf Rewards.} We treat this problem as a one-step decision making problem (Section \ref{sec:one-step}). Once the robot chooses a motion primitive to execute based on its sensor measurements, it receives a reward of 0 if the motion primitive results in a collision with an obstacle; if the motion primitive results in collision-free motion, the robot receives a reward of 1. The expected reward for this problem thus corresponds to the probability of safe (i.e., collision-free) motion. 

{\bf Sensor.} The robot is equipped with a depth sensor which provides distances along $n_\text{rays}=10$ rays. The sensor has a field of view of $90^\circ$ and a maximum range of $1.5$m.  We use the noise model for range finders described in Ch. 6.3 in the work by \cite{Thrun05} and consider two kinds of measurement errors: (i) errors due to failures to detect obstacles, and (ii) random noise in the reported distances. For each ray, there is a probability $p_\text{miss}=0.1$ that the sensor misses an obstacle and reports the maximum range ($1.5$m) instead. In the case that an obstacle is not missed, the distance reported along a given ray is sampled from a Gaussian with mean equal to the true distance along that ray and std. dev. equal to $\eta$. The noise for each ray is sampled independently. Overall, this is a non-Gaussian sensor model due to the combination of the two kinds of errors. 

{\bf Results: varying sensor noise.} We implement Algorithm~\ref{a:DP} (with KL divergence) using the sampling-based techniques described in Section \ref{sec:sampling}. We sample $20$K obstacle environments for upper bounding the open-loop rewards associated with each action. We also utilize $20$K batches (each of size $K=1000$) for upper bounding the mutual information using \eqref{eq:variational MI bound}. We utilize a version of Hoeffding's inequality (Theorem~\ref{thm:hoeffding}) presented in Appendix~\ref{app:hoeffding} to obtain an upper bound on $\revision{R_{0\to T}^\star}$ that holds with probability $1-\delta = 0.95$.

\begin{figure}[t]
\begin{center}
\includegraphics[width=0.99\columnwidth]{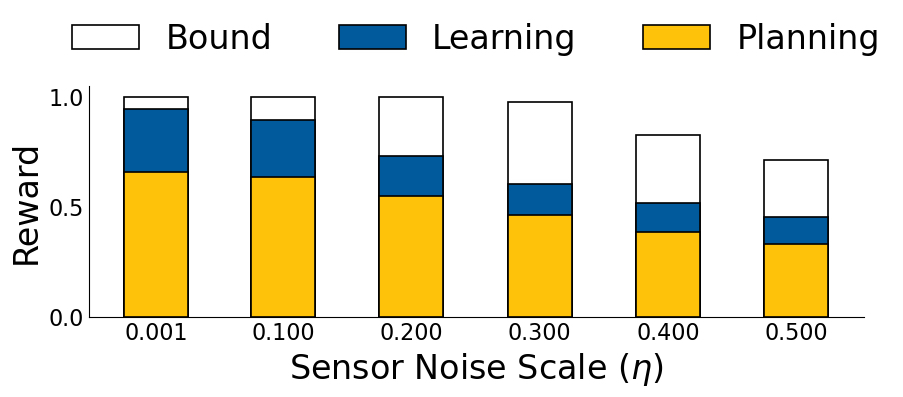}
\end{center} 
\vspace{-5pt}
\caption{\revision{Comparing sensors with varying levels of noise for the obstacle avoidance problem. Upper bounds on achievable expected rewards are compared with the expected rewards using (i) a learned neural network policy and (ii) a heuristic planner.} \label{fig:planning results}}
\vspace{-12pt}
\end{figure}

\begin{figure}[t]
\begin{center}
\includegraphics[width=0.99\columnwidth]{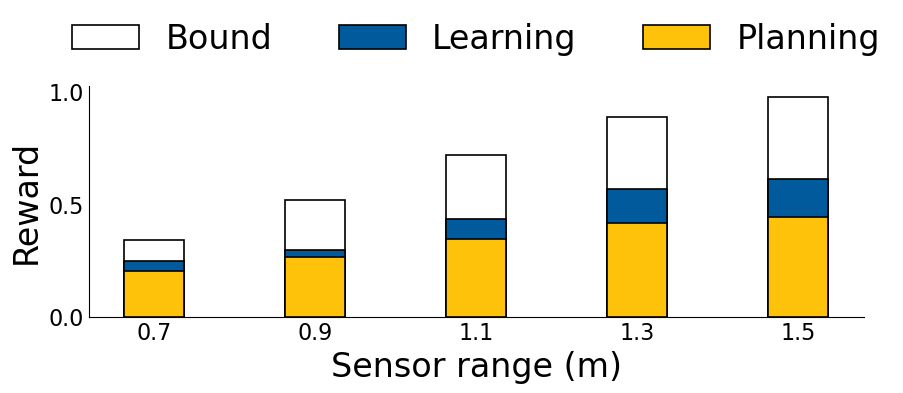}
\end{center} 
\vspace{-5pt}
\caption{\revision{Comparing sensors with varying distance ranges for the obstacle avoidance problem. Upper bounds on achievable expected rewards are compared with the expected rewards using (i) a learned neural network policy and (ii) a heuristic planner.} \label{fig:planning results range}}
\vspace{-12pt}
\end{figure}

Figure \ref{fig:planning results} shows the resulting upper bounds for different values of the observation noise standard deviation $\eta$. We compare these upper bounds with rewards achieved by neural network policies trained to perform this task; \revision{these rewards provide lower bounds on the highest achievable rewards for each $\eta$}. For each $\eta$, we sample 5000 training environments and corresponding sensor observations. For each environment, we generate a ten-dimensional training label by recording the minimum distance to the obstacles achieved by each motion primitive and passing the vector of distances through an (element-wise) softmax transformation. We use a cross-entropy loss to train a neural network that predicts the label of distances for each primitive given a sensor observation as input. We use two fully-connected layers with a \texttt{ReLu} nonlinearity; the output is passed through a softmax layer. At test-time, a given sensor observation is passed as input to the trained neural network; the motion primitive corresponding to the highest predicted distance is then executed. We estimate the expected reward achieved by the trained policy using 5000 test environments (unseen during training). Figure \ref{fig:planning results} plots the average of these expected rewards across five training seeds for each value of $\eta$ (the std. dev. across seeds is too small to visualize). \revision{Figure~\ref{fig:planning results} also presents rewards achieved by a heuristic planner that chooses the motion primitive with the largest estimated distance to the obstacles.} 

\revision{As expected, the rewards from both the learning- and planning-based approaches are lower than the fundamental limits computed using our approach.}
This highlights the fact that our upper bounds provide fundamental limits on performance that hold regardless of the size of the neural network, the network architecture, or algorithm used for synthesizing the policy. We also highlight the fact that the neural network policy for a sensor with noise $\eta=0.1$ achieves higher performance than the fundamental limit for a sensor with noise $\eta=0.4$ or $0.5$.  

\revision{{\bf Results: varying sensor range.} Next, we compare depth sensors with different distance ranges in Figure~\ref{fig:planning results range}. For this comparison, we fix $\eta = 0.3$ and $p_\text{miss} = 0.05$. Consistent with intuition, the upper bounds on achievable performance increase as the range of the sensors increase. The performance of the learning- and heuristic planning-based approaches also improve with the range, while remaining below the fundamental limits. We highlight that the neural network policy for a sensor with range 1.5m surpasses the fundamental limit for a sensor with range 0.9m.}

\revision{{\bf Results: varying sensor resolution.}} Finally, we apply our approach to compare two sensors with varying number $n_\text{rays}$ of rays along which the depth sensor provides distance estimates (Figure~\ref{fig:anchor}). For this comparison, we fix $\eta = 0.3$ and $p_\text{miss} = 0.05$. We compare two sensors with $n_\text{rays} = 50$ (Sensor 1) and $n_\text{rays} = 5$ (Sensor 2) respectively. The upper bound on expected reward computed using our approach (with confidence $1-\delta = 0.95$) for Sensor~2 is 0.79. A neural network policy for Sensor 1 achieves an expected reward of approximately 0.86, which surpasses the fundamental limit on performance for Sensor 2.

\section{Discussion and Conclusions}
\label{sec:conclusion}

We have presented an approach for establishing fundamental limits on performance for sensor-based robot control and policy learning. We defined a quantity that captures the amount of task-relevant information provided by a sensor; using a novel version of the generalized Fano inequality, we demonstrated that this quantity upper bounds the expected reward for one-step decision making problems. We developed a dynamic programming approach for extending this bound to multi-step problems. The resulting framework has potentially broad applicability to robotic systems with continuous state and observation spaces, nonlinear and stochastic dynamics, and non-Gaussian sensor models. 
Our numerical experiments demonstrate the ability of our approach to establish strong bounds on performance for such settings. In addition, we provided an application of our approach for comparing different sensors and establishing the superiority of one sensor over another for a given task. 

{\bf Challenges and future work.} There are a number of challenges and directions for future work associated with our approach. On the theoretical front,  
an interesting direction is to handle settings where the sensor model is inaccurate (in contrast to this paper, where we have focused on establishing fundamental limits given a particular sensor model). For example, one could potentially perform an adversarial search over a family of sensor models in order to find the model that results in the lowest bound on achievable performance.

On the algorithmic front, the primary challenges are: (i) efficient computation of bounds for longer time horizons, and (ii) extensions to continuous action spaces. As presented, Algorithm~\ref{a:DP} requires an enumeration over action sequences. Finding more computationally efficient versions of Algorithm~\ref{a:DP} is thus an important direction for future work. The primary bottleneck in extending our approach to continuous action spaces is the need to perform a supremization over actions when computing $\revision{\bar{R}_{t\to T}^{\perp \star}}$ in Algorithm~\ref{a:DP}. However, we note that any upper bound on $\revision{\bar{R}_{t\to T}^{\perp \star}}$ also leads to a valid upper bound on \revision{$R^\star_{0\to T}$}. Thus, one possibility is to use a Lagrangian relaxation \citep{Boyd04} to upper bound $\revision{\bar{R}_{t\to T}^{\perp \star}}$ in settings with continuous action spaces. \revision{Another current limitation in computing the bounds for systems with continuous state and observation spaces (Section~\ref{sec:sampling}) is the need for a probability density function for the sensor model $\sigma_t$ than can be numerically evaluated given any particular state-observation pair. This may be addressed by leveraging the broad range of numerical techniques that exist for mutual information estimation \citep{Poole19, Song19}.}

Our work also opens up a number of exciting directions for longer-term research. While we have focused on establishing fundamental limits imposed by imperfect sensing in this work, one could envision a broader research agenda that seeks to establish bounds on performance due to other limited resources (e.g., onboard computation or memory). One concrete direction is to combine the techniques presented here with information-theoretic notions of bounded rationality \citep{Tishby11, Ortega15, Pacelli21}. Finally, another exciting direction is to turn the impossibility results provided by our approach into certificates of \emph{robustness} against an adversary. Specifically, consider an adversary that can observe our robot's actions; if one could establish fundamental limits on performance \emph{for the adversary} due to its inability to infer the robot's internal state (and hence its future behavior) using past observations, this provides a certificate of robustness against \emph{any} adversary. This is reminiscent of a long tradition in cryptography of turning impossibility or hardness results into robust protocols for security \citep[Ch. 9]{Aumasson17}. 

\vspace{5pt}

Overall, we believe that the ability to establish fundamental limits on performance for robotic systems is crucial for establishing a science of robotics. We hope that the work presented here along with the indicated directions for future work represent a step in this direction.  

\quad

\section*{Acknowledgments}
The authors were partially supported by the NSF CAREER Award [$\#$2044149] and the Office of Naval Research [N00014-21-1-2803]. The authors would like to thank Alec Farid and David Snyder for helpful discussions on this work. 


\bibliographystyle{unsrtnat}
\balance
\bibliography{references}

\newpage
\onecolumn

\begin{appendices}
\section{Proofs}
\label{app:proofs}

\revision{\propmonotone*}
\begin{proof}
Recall that the $f$-inverse is defined using the following optimization problem with decision variable $p$:
\begin{equation}
\label{eq:f inverse opt}
    \mathbb D_f^{-1}(q | c) := \sup \ \{p \in [0,1] \ | \ \mathbb D_{f,\mathcal B}(p \| q) \leq c \}. 
\end{equation}
where
\begin{equation*}
\mathbb D_{f,\mathcal B}(p||q):= \mathbb D_f(\mathcal B_p||\mathcal B_q)=qf\left(\frac{p}{q}\right)+(1-q)f\left(\frac{1-p}{1-q}\right). \label{eq:f Ber}
\end{equation*}

The monotonicity condition consists of two parts\revision{, as stated above.}

To prove 1), we see that increasing $c$ loosens constraint to the optimization problem (\ref{eq:f inverse opt}), thus the $f$-inverse can only take values larger than or equal to the ones corresponding to smaller $c$'s.

To prove 2), \revision{we first show that $\mathbb D_{f,\mathcal B}(p||q)$ is monotonically non-increasing in $q$ for any fixed $p$.}
For notational simplicity, we drop the subscripts for $f$-divergence \revision{and} denote $\mathbb D_{f,\mathcal B}(p||q)$ by $ \mathbb D(p,q)$.

We note that \revision{$\mathbb D(p,q)$ is directionally differentiable, since $f$ is a one-dimensional convex function. }
For fixed $p$, the directional derivatives in the positive and negative $q$ directions are:
\begin{align*}
\mathbb D'_\pm(p,q\pm\delta) &=\frac{\partial}{\partial q_\pm}\left[qf\left(\frac{p}q{}\right)+(1-q)f\left(\frac{1-p}{1-q}\right)\right]\\
&=f\left(\frac{p}{q}\right)-f\left(\frac{1-p}{1-q}\right)+q\frac{\partial}{\partial q_\pm}f\left(\frac{p}{q}\right)+(1-q)\frac{\partial}{\partial q_\pm}f\left(\frac{1-p}{1-q}\right)\\
&=f(x)-f(y)+y f'_\pm(y)-xf'_\pm(x)\\
&=\underbrace{f(x)-f(y)+f'_\pm(x)(y-x)}_{(a)}+\underbrace{y(f'_\pm(y)-f'_\pm(x))}_{(b)}.
\end{align*}
    
\begin{itemize}
    \item In (a), 
    for all $0<t\leq 1$, we have by convexity of $f$,
    \begin{equation*}
        f(x+t(y-x))\leq (1-t)f(x)+tf(y).
    \end{equation*}
    Dividing both sides by $t$,
    \begin{equation*}
        f(y)-f(x)\geq \frac{f(x+t(y-x))-f(x)}{t}.
    \end{equation*}
    If we take the limit as $t\to0$ from the right side of the $x$-axis, then the right directional derivatives in the direction of $y-x$ \revision{is}:
    \begin{equation}
        \label{eq:right direction}
        f'_+(x)(y-x)=\lim_{t\to 0^+}\frac{f(x+t(y-x))-f(x)}{t} \leq f(y)-f(x).
    \end{equation}
    \revision{Similarly, for $-1\leq t < 0$, we have,
    \begin{align*}
        f(x+t(x-y))&=f(x+(-t)(y-x))\leq (1+t)f(x)-tf(y),\\
        (-t)(f(y)-f(x))&\geq f(x+t(x-y))-f(x),\\
        f(y)-f(x)&\geq \frac{f(x)-f(x+t(x-y))}{t}.
    \end{align*}
    The left directional derivative in the direction of $y-x$ becomes:
    \begin{equation}
        \label{eq:left direction}
        f'_-(x)(y-x)=\lim_{t\to 0^-}\frac{f(x)-f(x+t(x-y))}{t} \leq f(y)-f(x).
    \end{equation}
    Combining \eqref{eq:right direction} and \eqref{eq:left direction}, we see that} $(a) = f(x)-f(y)+f'_\pm(x)(y-x)\leq 0$. 
    \item In $(b)$, we know that $y=\frac{1-p}{1-q}>0$. We also know that the optimal value $p^\star$ for problem (\ref{eq:f inverse opt}) is greater than or equal to $q$. So, $x^\star = \frac{p^\star}{q}\geq 1, y^\star=\frac{1-p^\star}{1-q}\leq 1, x^\star\geq y^\star$. For the convex function $f$, the directional derivatives $f'_\pm$ are nondecreasing. Therefore, for $x^\star>y^\star, f'_\pm(x^\star)\geq f'_\pm(y^\star)$. It then follows that $(b)\leq0$.
\end{itemize}

\revision{Therefore, $\mathbb D'_\pm(p,q\pm\delta)\leq 0$, $\mathbb D(p,q)$ is non-increasing in $q$ for fixed $p$, and $q\leq q'\Rightarrow \mathbb D(p,q)\geq \mathbb D(p,q')$. Thus, for all fixed $c\geq 0, \mathbb D(p,q)\leq c$ implies $\mathbb D(p,q')\leq c$. Noting that this is the constraint for the optimization problem \eqref{eq:f inverse opt}, we see that any feasible point $p$ for $\mathbb D_f^{-1}(q|c)$ is also a feasible point for $\mathbb D_f^{-1}(q'|c)$. In other words, $q\leq q'$ implies $\mathbb D_f^{-1}(q|c)\leq\mathbb D_f^{-1}(q'|c)$, or $f$-inverse is monotonically non-decreasing in $q$ for fixed $c$.}

\end{proof}


\quad

\thmonestep*

\begin{proof}
For a given policy $\pi_0$, define:
\begin{equation*}
    \revision{R_{0\to 1}} := \underset{p_0(s_0)}{\mathbb{E}} \ \underset{\sigma_0(o_0|s_0)}{\mathbb{E}}[r_0(s_0, \pi_0(o_0))],
\end{equation*}
and
\begin{equation*}
    \revision{\tilde{R}_{0\to 1}} := \underset{p_0(s_0)}{\mathbb{E}} \ \underset{q(o_0)}{\mathbb{E}}[r_0(s_0, \pi_0(o_0))].
\end{equation*}

The only difference between $R$ and $\revision{\tilde R_{0\to 1}}$ is that the observations $o_0$ in $\revision{\tilde R_{0\to 1}}$ are drawn from a state-\emph{independent} distribution $q$. 

Now, assuming bounded rewards (Assumption~\ref{ass:bounded rewards}), we have:
\begin{align}
    \fb(\revision{R_{0\to 1}}\| \revision{\tilde R_{0\to 1}}) &= \fb \Bigg{(} \underset{p_0(s_0)}{\mathbb{E}} \ \underset{\sigma_0(o_0|s_0)}{\mathbb{E}}[r_0(s_0, \pi_0(o_0))] \ \Big{\|} \ \underset{p_0(s_0)}{\mathbb{E}} \ \underset{q(o_0)}{\mathbb{E}}[r_0(s_0, \pi_0(o_0))] \Bigg{)} \nonumber\\
    &\leq \underset{p_0(s_0)}{\mathbb{E}} \ \fb \Bigg{(} \underset{\sigma_0(o_0|s_0)}{\mathbb{E}}[r_0(s_0, \pi_0(o_0))] \ \Big{\|} \ \underset{q(o_0)}{\mathbb{E}}[r_0(s_0, \pi_0(o_0))] \Bigg{)} \nonumber\\
    &\leq \underset{p_0(s_0)}{\mathbb{E}} \ \mathbb D_f \Big{(} \sigma_0(o_0|s_0) \ \Big{\|} \ q(o_0) \Big{)}. \label{eq:one step DPI}
\end{align}

The first inequality above follows from Jensen's inequality, while the second follows from the data processing inequality (see Corollary 2 by \cite{Gerchinovitz20} for the specific version). \revision{For notational simplicity, we denote the right-hand side of equation \eqref{eq:one step DPI} by $\mathbb J^q:=\underset{p_0(s_0)}{\mathbb{E}} \ \mathbb D_f \Big{(} \sigma_0(o_0|s_0) \ \Big{\|} \ q(o_0) \Big{)}$. Then, inverting the bound using the $f$-inverse:
\begin{equation*}
    R_{0\to 1}\leq\mathbb D^{-1}_f(\tilde R_{0\to 1}| \mathbb J^q).
\end{equation*}}

\revision{Taking supremum over policies on both sides gives:
\begin{equation*}
\sup_{\pi_0}\,R_{0\to 1}\leq\sup_{\pi_0}\mathbb D_f^{-1}(\tilde R_{0\to 1}| \mathbb J^q).
\end{equation*}}

\revision{On the RHS, $\tilde R_{0\to 1}$ depends on $\pi_0$ but $\mathbb J^q(\sigma_0, p_0, q)$ is independent of policy. Thus, from the monotonicity of the $f$-inverse (Proposition \ref{prop:f inverse properties}) and by definition of the highest achievable expected reward, we have:
\begin{equation}
R^\star_{0\to 1}(\sigma_0;r_0):=\sup_{\pi_0}\,R_{0\to 1}\leq\mathbb D_f^{-1}(\sup_{\pi_0}\tilde R_{0\to 1}| \mathbb J^q).\label{eq:J-dependent bound}
\end{equation}}

\revision{Next, we show that $\sup_{\pi_0}\tilde R_{0\to 1}$ is equal to the highest expected reward achieved by open-loop policies, and thus is independent of the distribution $q$. Using the Fubini-Tonelli theorem, we see:}
\begin{align*}
    \sup_{\pi_0} \ \revision{\tilde R_{0\to 1}} &= \sup_{\pi_0} \ \underset{p_0(s_0)}{\mathbb{E}} \ \underset{q(o_0)}{\mathbb{E}}[r_0(s_0, \pi_0(o_0))] \\
    &= \sup_{\pi_0} \  \underset{q(o_0)}{\mathbb{E}} \ \underset{p_0(s_0)}{\mathbb{E}}[r_0(s_0, \pi_0(o_0))] \\
    &\leq \underset{q(o_0)}{\mathbb{E}} \ \sup_{\pi_0} \ \underset{p_0(s_0)}{\mathbb{E}}[ r_0(s_0, \pi_0(o_0))] \\
    &= \underset{q(o_0)}{\mathbb{E}} \ \sup_{a_0} \ \underset{p_0(s_0)}{\mathbb{E}}[ r_0(s_0, a_0)] \\
    &= \sup_{a_0} \ \underset{p_0(s_0)}{\mathbb{E}}[ r_0(s_0, a_0)].
\end{align*}

Since open-loop actions are special cases of policies, we also have:
\begin{equation*}
    \sup_{\pi_0} \ \revision{\tilde R_{0\to 1}} = \sup_{\pi_0} \ \underset{p_0(s_0)}{\mathbb{E}} \ \underset{q(o_0)}{\mathbb{E}}[r_0(s_0, \pi_0(o_0))] \ \geq \ \sup_{a_0} \ \underset{p_0(s_0)}{\mathbb{E}}[ r_0(s_0, a_0)].
\end{equation*}

As a result, we see that:
\begin{align*}
    \sup_{\pi_0} \ \revision{\tilde R_{0\to 1}} = \sup_{a_0} \ \underset{p_0(s_0)}{\mathbb{E}}[ r_0(s_0, a_0)] = \revision{R_{0\to 1}^\perp},
\end{align*}
where \revision{$R_{0\to 1}^\perp$} is as defined in \eqref{eq:R0 perp}. \revision{Now, taking the infimum over the state-independent distribution $q$ on both sides of equation \eqref{eq:J-dependent bound}:
\begin{equation}
R_{0\to 1}^\star (\sigma_0;r_0)= \inf_q R_{0\to 1}^\star(\sigma_0;r_0)\leq\mathbb D_f^{-1}(R_{0\to 1}^\perp| \inf_q\mathbb J^q).\label{eq:one step f inverse bound}
\end{equation}}

\revision{From the definition of $f$-informativity \citep{csiszar_class_1972}, we note that: 
\begin{equation*}
    \I_f(o_0;s_0) = \inf_q \ \underset{p_0(s_0)}{\mathbb{E}} \ \mathbb D_f \Big{(} \sigma_0(o_0|s_0) \ \Big{\|} \ q(o_0) \Big{)}=\inf_q \mathbb J^q. 
\end{equation*}}

\revision{Combining this with \eqref{eq:one step f inverse bound}, we obtain the desired result:
\begin{equation*}
    R_{0\to 1}^\star(\sigma_0; r_0) \leq \finv \Big{(} R_{0\to 1}^\perp | \ \I_f(o_0; s_0) \Big{)} =: \tau(\sigma_0; r_0).
\end{equation*}}

\end{proof}


\proprecursion*
\begin{proof}
The proof follows a similar structure to that of Theorem \ref{thm:single-step bound}. 

First, note that \revision{$R_{t\to T}$} defined in \eqref{eq:Rt} can be written as:
\begin{equation}
    \revision{R_{t\to T}} = \underset{\substack{s_t| \\ a_{0:t-1}}}{\mathbb E} \ \underset{o_t|s_t}{\mathbb E} \ \underset{\substack{s_{t+1:T-1}, o_{t+1:T-1} | \\ s_t,o_t}}{\mathbb E} \ \Bigg{[} \sum_{k=t}^{T-1} r_k(s_k, \revision{\pi^k_t}(o_{t:k})) \Bigg{]}. \nonumber
\end{equation}

Define:
\begin{equation}
    \revision{\tilde{R}_{t\to T}} := \underset{\substack{s_t| \\ a_{0:t-1}}}{\mathbb E} \ \underset{\substack{q(o_t)}}{\mathbb E} \ \underset{\substack{s_{t+1:T-1}, o_{t+1:T-1} | \\ s_t,o_t}}{\mathbb E} \ \Bigg{[} \sum_{k=t}^{T-1} r_k(s_k, \revision{\pi^k_t}(o_{t:k})) \Bigg{]}. \nonumber
\end{equation}

The only difference between $\revision{\tilde R_{t\to T}}$ and \revision{$R_{t\to T}$} is that the observations $o_t$ in $\revision{\tilde R_{t\to T}}$ are drawn from a state-\emph{independent} distribution $q$.

For the sake of notational simplicity, we will assume that \revision{$R_{t\to T}$} and $\revision{\tilde R_{t\to T}}$ have been normalized to be within $[0,1]$ by scaling with $1/(T-t)$. The desired result \eqref{eq:history-dependent induction} then follows from the bound we prove below by simply rescaling with $(T-t)$. 

Now, 
\begin{small}
\begin{align}
    \fb(\revision{R_{t\to T}} \| \revision{\tilde R_{t\to T}}) = \ &\fb\Bigg{(} \underset{\substack{s_t| \\ a_{0:t-1}}}{\mathbb E} \ \underset{\substack{o_t | \\ s_t}}{\mathbb E} \ \underset{\substack{s_{t+1:T-1}, o_{t+1:T-1} | \\ s_t,o_t}}{\mathbb E} \ \Bigg{[} \sum_{k=t}^{T-1} r_k(s_k, \revision{\pi_t^k}(o_{t:k})) \Bigg{]} \ \Big{\|} \  \underset{\substack{s_t| \\ a_{0:t-1}}}{\mathbb E} \ \underset{\substack{q(o_t)}}{\mathbb E} \ \underset{\substack{s_{t+1:T-1}, o_{t+1:T-1} | \\ s_t,o_t}}{\mathbb E} \ \Bigg{[} \sum_{k=t}^{T-1} r_k(s_k, \revision{\pi_t^k}(o_{t:k})) \Bigg{]} \Bigg{)} \nonumber \\
    &\leq \underset{\substack{s_t| \\ a_{0:t-1}}}{\mathbb E} \ \fb \Bigg{(} \underset{\substack{o_t | \\ s_t}}{\mathbb E} \ \ \underset{\substack{s_{t+1:T-1}, o_{t+1:T-1} | \\ s_t,o_t}}{\mathbb E} \ \Bigg{[} \sum_{k=t}^{T-1} r_k(s_k, \revision{\pi_t^k}(o_{t:k})) \Bigg{]} \ \Big{\|} \  \underset{\substack{q(o_t)}}{\mathbb E} \ \ \underset{\substack{s_{t+1:T-1}, o_{t+1:T-1} | \\ s_t,o_t}}{\mathbb E} \ \Bigg{[} \sum_{k=t}^{T-1} r_k(s_k, \revision{\pi_t^k}(o_{t:k})) \Bigg{]} \Bigg{)} \nonumber \\
    &\leq \underset{\substack{s_t| \\ a_{0:t-1}}}{\mathbb E} \ \revision{\mathbb D_f}\Big{(} \sigma_t(o_t|s_t) \ \Big{\|} \ q(o_t)\Big{)}. \label{eq:f bound}
\end{align}
\end{small}
The first inequality above follows from Jensen's inequality, while the second follows from the data processing inequality (see Corollary 2 by \cite{Gerchinovitz20} for the specific version). \revision{We denote the RHS of \eqref{eq:f bound} by $\mathbb J^q:=\underset{\substack{s_t| \\ a_{0:t-1}}}{\mathbb E} \ \mathbb D_f\Big{(} \sigma_t(o_t|s_t) \ \Big{\|} \ q(o_t)\Big{)}$ and invert the bound:
\begin{equation*}
    \revision{R_{t\to T}} \leq \finv \Big{(} \revision{\tilde R_{t\to T}} \ | \ \ \mathbb J^q \Big{)}.
\end{equation*}}

\revision{Taking supremum over policies on both sides gives:
\begin{equation*}
    \sup_{\pi^t_t, \dots, \pi_t^{T-1}} \ R_{t\to T} \ \leq \ \sup_{\pi^t_t, \dots, \pi_t^{T-1}} \ \finv \Big{(} \tilde{R}_{t\to T} \ | \ \ \mathbb J^q \Big{)}.
\end{equation*}}


Notice that the LHS is precisely the quantity we are interested in upper bounding in Proposition \ref{prop:recursion}.
From the monotonicity of the $f$-inverse (\revision{Proposition \ref{prop:f inverse properties}}), we have: 
\begin{equation}
\label{eq:Rt bound}
    \sup_{\pi^t_t, \dots, \revision{\pi_t^{T-1}}} \ \revision{R_{t\to T}} \  \leq \ \finv \Big{(} \revision{\tilde{R}_{t\to T}^\star} \ | \ \ \revision{\mathbb J^q} \Big{)},
\end{equation}
where
\begin{equation*}
    \revision{\tilde R_{t\to T}^\star} := \sup_{\pi^t_t, \dots, \revision{\pi_t^{T-1}}} \ \revision{\tilde R_{t\to T}} = \sup_{\pi^t_t, \dots, \revision{\pi_t^{T-1}}} \ \underset{\substack{s_t| \\ a_{0:t-1}}}{\mathbb E} \ \underset{\substack{q(o_t)}}{\mathbb E} \ \underset{\substack{s_{t+1:T-1}, o_{t+1:T-1} | \\ s_t,o_t}}{\mathbb E} \ \Bigg{[} \sum_{k=t}^{T-1} r_k(s_k, \revision{\pi_t^k}(o_{t:k})) \Bigg{]}.
\end{equation*}

\revision{We then show that $\tilde R_{t\to T}^\star$ is the highest expected total reward achieved by open-loop policies, and thus independent of $q$. U}sing the Fubini-Tonelli theorem:
\begin{align}
    &\sup_{\pi^t_t, \dots, \revision{\pi_t^{T-1}}} \ \underset{\substack{s_t| \\ a_{0:t-1}}}{\mathbb E} \ \underset{\substack{q(o_t)}}{\mathbb E} \ \underset{\substack{s_{t+1:T-1}, o_{t+1:T-1} | \\ s_t,o_t}}{\mathbb E} \ \Bigg{[} \sum_{k=t}^{T-1} r_k(s_k, \revision{\pi_t^k}(o_{t:k})) \Bigg{]} \nonumber\\
    = &\sup_{\pi^t_t, \dots, \revision{\pi_t^{T-1}}} \ \underset{\substack{q(o_t)}}{\mathbb E} \ \underset{\substack{s_t| \\ a_{0:t-1}}}{\mathbb E}  \ \underset{\substack{s_{t+1:T-1}, o_{t+1:T-1} | \\ s_t,o_t}}{\mathbb E} \ \Bigg{[} \sum_{k=t}^{T-1} r_k(s_k, \revision{\pi_t^k}(o_{t:k})) \Bigg{]} \nonumber\\
    = &\sup_{\pi^t_{t+1}\dots, \revision{\pi_t^{T-1}}} \ \Bigg{[} \sup_{\pi^t_t} \ \underset{\substack{q(o_t)}}{\mathbb E} \ \underset{\substack{s_t| \\ a_{0:t-1}}}{\mathbb E}  \ \underset{\substack{s_{t+1:T-1}, o_{t+1:T-1} | \\ s_t,o_t}}{\mathbb E} \ \Bigg{[} \sum_{k=t}^{T-1} r_k(s_k, \revision{\pi_t^k}(o_{t:k})) \Bigg{]} \ \Bigg{]} \nonumber\\
    \leq &\sup_{\pi^t_{t+1}\dots, \revision{\pi_t^{T-1}}} \ \Bigg{[} \underset{\substack{q(o_t)}}{\mathbb E} \ \sup_{\pi^t_t} \ \underset{\substack{s_t| \\ a_{0:t-1}}}{\mathbb E}  \ \underset{\substack{s_{t+1:T-1}, o_{t+1:T-1} | \\ s_t,o_t}}{\mathbb E} \ \Bigg{[} \sum_{k=t}^{T-1} r_k(s_k, \revision{\pi_t^k}(o_{t:k})) \Bigg{]} \ \Bigg{]}. \label{eq:Rt tilde star bound}
\end{align}

Notice that:
\begin{align}
    &\underset{\substack{q(o_t)}}{\mathbb E} \ \sup_{\pi^t_t} \ \underset{\substack{s_t| \\ a_{0:t-1}}}{\mathbb E}  \ \underset{\substack{s_{t+1:T-1}, o_{t+1:T-1} | \\ s_t,o_t}}{\mathbb E} \ \Bigg{[} \sum_{k=t}^{T-1} r_k(s_k, \revision{\pi_t^k}(o_{t:k})) \Bigg{]} \nonumber\\
    =  &\underset{\substack{q(o_t)}}{\mathbb E} \ \sup_{\pi^t_t} \ \underset{\substack{s_t| \\ a_{0:t-1}}}{\mathbb E}  \  \Bigg{[} r_t(s_t, \pi^t_t(o_t)) + \underset{\substack{s_{t+1}, o_{t+1}| \\ s_t,\pi^t_t(o_t)}}{\mathbb E} \Bigg{[} r_{t+1}(s_{t+1}, \revision{\pi_{t}^{t+1}}(o_{t:t+1})) + \dots \Bigg{]} \dots \Bigg{]} \label{eq:sup policy}\\
    = &\underset{\substack{q(o_t)}}{\mathbb E} \ \underbrace{\sup_{a_t} \ \underset{\substack{s_t| \\ a_{0:t-1}}}{\mathbb E}  \  \Bigg{[} r_t(s_t, a_t) + \underset{\substack{s_{t+1}, o_{t+1}| \\ s_t,a_t}}{\mathbb E} \Bigg{[} r_{t+1}(s_{t+1}, \pi^{t+1}_{t+1}(o_{t+1})) + \dots \Bigg{]} \dots \Bigg{]}}_\text{Does not depend on $o_t$} \label{eq:sup action}\\
    = &\sup_{a_t} \ \underset{\substack{s_t| \\ a_{0:t-1}}}{\mathbb E}  \  \Bigg{[} r_t(s_t, a_t) + \underset{\substack{s_{t+1}, o_{t+1}| \\ s_t,a_t}}{\mathbb E} \Bigg{[} r_{t+1}(s_{t+1}, \pi^{t+1}_{t+1}(o_{t+1})) + \dots \Bigg{]} \dots \Bigg{]}.\nonumber
\end{align}

Here, \eqref{eq:sup action} follows  \eqref{eq:sup policy} since $q$ is a fixed distribution that does not depend on the state. 

We thus see that \eqref{eq:Rt tilde star bound} equals:
\begin{align*}
    &\sup_{\pi^t_{t+1}\dots, \revision{\pi_t^{T-1}}} \ \underbrace{\Bigg{[} \sup_{a_t} \ \underset{\substack{s_t| \\ a_{0:t-1}}}{\mathbb E}  \  \Bigg{[} r_t(s_t, a_t) + \underset{\substack{s_{t+1}, o_{t+1}| \\ s_t,a_t}}{\mathbb E} \Bigg{[} r_{t+1}(s_{t+1}, \pi^{t+1}_{t+1}(o_{t+1})) + \dots \Bigg{]} \dots \Bigg{]} \Bigg{]}}_\text{Does not depend on $o_t$} \\
    = &\sup_{\pi^{t+1}_{t+1}\dots, \revision{\pi_{t+1}^{T-1}}} \ \Bigg{[} \sup_{a_t} \ \underset{\substack{s_t| \\ a_{0:t-1}}}{\mathbb E}  \  \Bigg{[} r_t(s_t, a_t) + \underset{\substack{s_{t+1}, o_{t+1}| \\ s_t,a_t}}{\mathbb E} \Bigg{[} r_{t+1}(s_{t+1}, \pi^{t+1}_{t+1}(o_{t+1})) + \dots \Bigg{]} \dots \Bigg{]} \Bigg{]} \\
    = &\sup_{\pi^{t+1}_{t+1}\dots, \revision{\pi_{t+1}^{T-1}}} \ \Bigg{[} \sup_{a_t} \ \underset{\substack{s_t| \\ a_{0:t-1}}}{\mathbb E} \Bigg{[} r_t(s_t, a_t) \Bigg{]} + \underset{\substack{s_{t+1:T-1}, o_{t+1:T-1} | \\ a_{0:t}}}{\mathbb E} \ \Bigg{[} \sum_{k=t+1}^{T-1} r_k(s_k, \revision{\pi_{t+1}^k}(o_{t+1:k})) \Bigg{]} \Bigg{]} \\
     = &\sup_{\pi^{t+1}_{t+1}\dots, \revision{\pi_{t+1}^{T-1}}} \ \revision{R_{t\to T}^\perp} \\
     = & \ \revision{R_{t\to T}^{\perp \star}}.
\end{align*}

We have thus proved that $\revision{\tilde R_{t\to T}^\star} \leq \revision{R_{t\to T}^{\perp \star}}$ (indeed, since open-loop policies are special cases of feedback policies, we also have $\revision{\tilde R_{t\to T}^\star} \geq \revision{R_{t\to T}^{\perp \star}}$ and thus $\revision{\tilde R_{t\to T}^\star} = \revision{R_{t\to T}^{\perp \star}}$). \revision{The RHS of \eqref{eq:Rt bound} then becomes:
\begin{equation*}
    \sup_{\pi^t_t, \dots, \revision{\pi_t^{T-1}}} \ \revision{R_{t\to T}} \  \leq \ \finv \Big{(} \revision{R_{t\to T}^{\perp \star}} \ | \ \ \revision{\mathbb J^q} \Big{)}.
\end{equation*}}

\revision{Now, taking the infimum over the state-independent distribution $q$ on both sides and using monotonicity of $f$-inverse (proposition \ref{prop:f inverse properties}):
\begin{equation}
    \sup_{\pi^t_t, \dots, \pi_t^{T-1}} \ R_{t\to T} \  = \inf_q \sup_{\pi^t_t, \dots,\pi_t^{T-1}} \ R_{t\to T} \ \leq \ \finv \Big{(} R_{t\to T}^{\perp \star} \ | \ \ \inf_q\mathbb J^q \Big{)}.\label{eq:prop2 inf J}
\end{equation}}

From the \revision{definition} of $f$-informativity, \citep{csiszar_class_1972}, we note that: 
\begin{equation*}
    \I_f(o_t;s_t) := \inf_q \ \underset{\substack{s_t| \\ a_{0:t-1}}}{\mathbb E} \mathbb D_f\Big{(}\sigma_t(o_t|s_t) \| q(o_t) \Big{)} \revision{=\inf_q\mathbb J^q}.
\end{equation*}

\revision{Combining with \eqref{eq:prop2 inf J}, we obtain the desired result:}

\begin{equation*}
\label{eq:sup Rt bound}
\sup_{\pi^t_t, \dots, \revision{\pi_t^{T-1}}} \ \revision{R_{t\to T}} \  \leq \ \finv \Big{(} \revision{R_{t\to T}^{\perp \star}} \ | \ \ \I_f(o_t; s_t) \Big{)}.
\end{equation*}
\end{proof}

\thmmultistep*
\begin{proof}
Using (backwards) induction, we prove that for all $t=T-1,\dots,0$, 
\begin{equation}
\label{eq:algorithm bound}
    \sup_{\pi^t_t, \dots, \revision{\pi_t^{T-1}}} \ \revision{R_{t\to T}} \leq \bar{R}_t(a_{0:t-1}), \ \forall a_{0:t-1}.
\end{equation}

Thus, in particular, 
\begin{equation*}
    \revision{R_{0\to T}^\star} = \sup_{\pi^0_0, \dots, \revision{\pi_0^{T-1}}} \revision{\ R_{0\to T} \leq \bar{R}_{0\to T}}.
\end{equation*}

We prove \eqref{eq:algorithm bound} by backwards induction starting from $t=T-1$. In particular, Proposition \ref{prop:recursion} leads to the inductive step. 

We first prove the base step of induction using $t=T-1$. Using \eqref{eq:history-dependent induction}, we see:
\begin{equation}
\label{eq:induction base}
        \sup_{\pi^{T-1}_{T-1}} \ \revision{R_{T-1\to T}} \  \leq \ \finv \Big{(} \revision{R_{T-1\to T}^{\perp \star}} \ | \ \ \I_f(o_{T-1}; s_{T-1}) \Big{)}.
\end{equation}

By definition (see \eqref{eq:Rt perp star definition}), 
\begin{align}
    \revision{R_{T-1\to T}^{\perp \star}} &= \sup_{a_{T-1}} \ \underset{s_{T-1}|a_{0:T-2}}{\mathbb E} \Big{[} r_{T-1}(s_{T-1}, a_{T-1}) \Big{]} + \underset{= 0}{\underbrace{\revision{R_{t\to T}}}} \nonumber \\
    &= \sup_{a_{T-1}} \ \underset{s_{T-1}|a_{0:T-2}}{\mathbb E} \Big{[} r_{T-1}(s_{T-1}, a_{T-1}) \Big{]} \nonumber \\
    &= \revision{\bar{R}_{T-1\to T}^{\perp \star}}. \nonumber
\end{align}

Combining this with \eqref{eq:induction base} and the monotonicity of the $f$-inverse (Proposition \ref{prop:f inverse properties}), we see:
\begin{align*}
    \sup_{\pi^{T-1}_{T-1}} \ \revision{R_{T-1\to T}} \  &\leq \ \finv \Big{(} \revision{\bar{R}_{T-1\to T}^{\perp \star}} \ | \ \ \I_f(o_{T-1}; s_{T-1}) \Big{)} \\
    &= \revision{\bar{R}_{T-1\to T}}(a_{0:T-2}).
\end{align*}

Next, we prove the induction step. Suppose it is the case that for $t \in \{0,\dots,T-2\}$, we have
\begin{equation}
\label{eq:induction hypothesis}
    \sup_{\pi^{t+1}_{t+1}, \dots, \revision{\pi_{t+1}^{T-1}}} \ \revision{R_{t+1\to T}} \leq \revision{\bar R_{t+1\to T}}(a_{0:t}).
\end{equation}

We then need to show that
\begin{equation}
\label{eq:induction WTS}
    \sup_{\pi^{t}_{t}, \dots, \revision{\pi_{t}^{T-1}}} \ \revision{R_{t\to T}} \leq \revision{\bar R_{t\to T}}(a_{0:t-1}).
\end{equation}

To prove this, we first observe that
\begin{align*}
    \revision{R_{t\to T}^{\perp \star}} &:= \sup_{\pi_{t+1}^{t+1},\dots,\revision{\pi_{t+1}^{T-1}}} \  \sup_{a_t} \ \Bigg{[} \underset{s_t|a_{0:t-1}}{\mathbb E} \Big{[} r_t(s_t, a_t) \Big{]} + \revision{R_{t+1\to T}} \Bigg{]} \nonumber \\
    &= \sup_{a_t} \ \Bigg{[} \underset{s_t|a_{0:t-1}}{\mathbb E} \Big{[} r_t(s_t, a_t) \Big{]} + \sup_{\pi_{t+1}^{t+1},\dots,\revision{\pi_{t+1}^{T-1}}} \ \revision{R_{t+1\to T}} \Bigg{]}. \nonumber
\end{align*}

Combining this with the induction hypothesis \eqref{eq:induction hypothesis}, we see
\begin{align*}
    \revision{R_{t\to T}^{\perp \star}} &\leq \sup_{a_t} \ \Bigg{[} \underset{s_t|a_{0:t-1}}{\mathbb E} \Big{[} r_t(s_t, a_t) \Big{]} + \revision{\bar R_{t+1\to T}}(a_{0:t}) \Bigg{]} \\
    &=: \revision{\bar{R}_{t\to T}^{\perp \star}}.
\end{align*}

Finally, combining this with \eqref{eq:history-dependent induction} and the monotonicity of the $f$-inverse (Prop. \ref{prop:f inverse properties}), we obtain the desired result \eqref{eq:induction WTS}:
\begin{align*}
  \sup_{\pi^t_t, \dots, \revision{\pi_t^{T-1}}} \ \revision{R_{t\to T}} \  &\leq \ (T-t) \cdot \finv \Bigg{(} \frac{\revision{\bar{R}_{t\to T}^{\perp \star}}}{T-t} \ | \ \ \I_f(o_t; s_t) \Bigg{)} \\
  &= \revision{\bar R_{t\to T}}(a_{0:t-1}).
\end{align*}
\end{proof}

\section{Chernoff-Hoeffding Bound}
\label{app:hoeffding}

In our numerical examples (Section \ref{sec:examples}), we utilize a slightly tighter version of Hoeffding's inequality than the one presented in Theorem \ref{thm:hoeffding}. In particular, we use the following Chernoff-Hoeffding inequality (see Theorem 5.1 by \cite{Mulzer18}).

\begin{theorem}[Chernoff-Hoeffding inequality \citep{Mulzer18}]
\label{thm:chernoff hoeffding}
Let $z$ be a random variable bounded within $[0,1]$, and let $z_1,\dots, z_n$ denote i.i.d. samples. Then, with probability at least $1 - \delta$ (over the sampling of $z_1,\dots,z_n$), the following bound holds with probability at least $1 - \delta$:
\begin{equation}
\label{eq:chernoff hoeffding f bound}
\mathbb D \Bigg{(} \frac{1}{n}\sum_{i=1}^n z_i \ \Big{\|} \ \mathbb{E}[z])\Bigg{)} \leq \frac{\log(2/\delta)}{n}. 
\end{equation}
\end{theorem}

We can obtain an upper bound on $\mathbb{E}[z]$ using \eqref{eq:chernoff hoeffding f bound} as follows:
\begin{equation}
\label{eq:chernoff hoeffding bound}
\mathbb{E}[z] \leq \ \sup \ \Bigg{\{} p \in [0,1) \ \Bigg{|} \ \fb \Big{(} \frac{1}{n}\sum_{i=1}^n z_i \ \Big{\|} \ p\Big{)} \leq \frac{\log(2/\delta)}{n} \Bigg{\}}.
\end{equation}
The optimization problem in the RHS of \eqref{eq:chernoff hoeffding bound} is analogous to the $f$-inverse defined in Section \ref{sec:kl inverse}, and can be thought of as a ``right" $f$-inverse (instead of a ``left" $f$-inverse). In the case of KL divergence, (similar to the $f$-inverse in Section \ref{sec:kl inverse}), we can solve the optimization problem in \eqref{eq:chernoff hoeffding bound} using geometric programming.

\end{appendices}

\end{document}